\documentclass[10pt]{article} 
\usepackage[preprint]{tmlr}   

\usepackage[utf8]{inputenc}
\usepackage[T1]{fontenc}
\usepackage{url}
\usepackage{booktabs}
\usepackage{amsfonts}
\usepackage{amsmath}
\usepackage{amsthm}
\usepackage{amssymb}
\usepackage{nicefrac}
\usepackage{microtype}
\usepackage{enumitem}
\usepackage{xcolor}
\usepackage{graphicx}
\usepackage{algorithm}
\usepackage{algpseudocode}
\usepackage{multirow}
\usepackage{mathtools}
\usepackage{tikz}
\usetikzlibrary{arrows.meta,positioning,calc,fit,backgrounds}

\newcommand{\ptopofl}{\textsc{pTopoFL}}
\newcommand{\R}{\mathbb{R}}
\newcommand{\E}{\mathbb{E}}
\newcommand{\calC}{\mathcal{C}}
\newcommand{\calD}{\mathcal{D}}
\newcommand{\Id}{\mathrm{Id}}
\newcommand{\loc}{\mathrm{loc}}
\newcommand{\Vtot}{V_{\mathrm{tot}}}
\newcommand{\Vin}{V_{\mathrm{in}}}
\newcommand{\Vbet}{V_{\mathrm{bet}}}
\newcommand{\norm}[1]{\left\lVert #1 \right\rVert}
\newcommand{\pinorm}[1]{\left\lVert #1 \right\rVert_{\pi}}
\DeclareMathOperator*{\argmin}{arg\,min}
\DeclareMathOperator{\med}{med}
\DeclareMathOperator{\MAD}{MAD}
\DeclareMathOperator{\tr}{tr}

\newtheorem{theorem}{Theorem}
\newtheorem{proposition}[theorem]{Proposition}
\newtheorem{lemma}[theorem]{Lemma}
\newtheorem{corollary}[theorem]{Corollary}
\theoremstyle{definition}

\newtheorem{assumption}{Assumption}

\theoremstyle{remark}

\IfFileExists{results_v2/tables/results_macros.tex}{\newcommand{\resdir}{results_v2}}{\newcommand{\resdir}{../results_v2}}
\IfFileExists{\resdir/tables/results_macros.tex}{
\newcommand{\AblAlphaOne}{0.455}
\newcommand{\AblAlphaZero}{0.377}
\newcommand{\AblFull}{0.379}
\newcommand{\AblNoClust}{0.455}
\newcommand{\AblNoTrust}{0.378}
\newcommand{\AblUnif}{0.378}
\newcommand{\AblVoneAlpha}{0.405}
\newcommand{\AlignAlARI}{0.95}
\newcommand{\AlignAlARIsub}{0.87}

\newcommand{\AlignCoARI}{-0.02}

\newcommand{\AlignCoDitto}{0.378}
\newcommand{\AlignCoFedavg}{0.477}

\newcommand{\AlignCoIfca}{0.360}

\newcommand{\AlignCoPtopo}{0.469}
\newcommand{\AlignCoRhoOrac}{0.36}
\newcommand{\AlignCoRhoTopo}{0.86}
\newcommand{\AlignCorr}{0.87}
\newcommand{\AlignGeARI}{0.95}

\newcommand{\AlignGeGain}{-0.002}

\newcommand{\AlignGeRhoOrac}{0.74}

\newcommand{\AlignTaARI}{1.00}
\newcommand{\AlignTaARIsub}{-0.01}

\newcommand{\AlignTaFedavg}{0.506}

\newcommand{\AlignTaIfca}{0.377}
\newcommand{\AlignTaPacfl}{0.497}
\newcommand{\AlignTaPtopo}{0.378}

\newcommand{\ContAUC}{1.00}
\newcommand{\ContFPR}{0.00}
\newcommand{\ContTPR}{1.00}

\newcommand{\CostDescOne}{0.007}

\newcommand{\CostHoneNumpy}{0.65}

\newcommand{\DPFedavgLLinf}{0.423}
\newcommand{\DPFedavgLLmid}{0.438}
\newcommand{\DPFedavgMIAinf}{0.508}

\newcommand{\DPLocalLL}{0.335}

\newcommand{\DPPtopoARImid}{0.00}
\newcommand{\DPPtopoLLinf}{0.359}
\newcommand{\DPPtopoLLmid}{0.519}
\newcommand{\DPPtopoMIAinf}{0.507}

\newcommand{\DPPtopoNPLLmid}{0.379}

\newcommand{\DPclip}{5}
\newcommand{\DPepsMid}{8}
\newcommand{\DPepsOK}{5426}
\newcommand{\DPgapPos}{0.60}
\newcommand{\DPinterMed}{0.1912}
\newcommand{\DPlooMax}{0.0339}
\newcommand{\DPlooMed}{0.0229}
\newcommand{\DPsafeSigma}{0.0021}

\newcommand{\DPsigPhiMid}{1.95}
\newcommand{\DPsigmaOK}{0.0200}
\newcommand{\DPsteps}{200}
\newcommand{\HDsizes}{303/261/46/130}
\newcommand{\NumSeeds}{10}
\newcommand{\PCAWorseFedavg}{0.06}

\newcommand{\PCBWorseFedavg}{0.63}

\newcommand{\PCHDWorseFedavg}{0.33}
\newcommand{\PCHDWorseLocal}{0.45}
\newcommand{\PropInfClean}{1.00}
\newcommand{\PropInfMid}{0.39}
\newcommand{\ResAARI}{0.95}

\newcommand{\ResADittoLL}{0.398}

\newcommand{\ResAFedavgLL}{0.451}

\newcommand{\ResAFedproxLL}{0.450}

\newcommand{\ResAIfcaLL}{0.393}

\newcommand{\ResALocalLL}{0.396}
\newcommand{\ResANBetter}{4}

\newcommand{\ResAOracleLL}{0.377}

\newcommand{\ResAPacflLL}{0.384}

\newcommand{\ResAPfedmeLL}{0.446}
\newcommand{\ResAPtopoAUC}{0.885}
\newcommand{\ResAPtopoLL}{0.379}

\newcommand{\ResARandomLL}{0.441}

\newcommand{\ResARhoOrac}{0.45}
\newcommand{\ResARhoRand}{0.81}
\newcommand{\ResARhoTopo}{0.46}

\newcommand{\ResAScaffoldLL}{0.448}

\newcommand{\ResBDittoLL}{0.553}

\newcommand{\ResBFedavgLL}{0.590}

\newcommand{\ResBIfcaLL}{0.545}

\newcommand{\ResBPtopoLL}{0.592}

\newcommand{\ResBRhoRand}{0.88}
\newcommand{\ResBRhoTopo}{0.90}

\newcommand{\ResHDBest}{Ditto}

\newcommand{\ResHDDittoLL}{0.435}
\newcommand{\ResHDFedavgAUC}{0.864}
\newcommand{\ResHDFedavgLL}{0.463}

\newcommand{\ResHDNBetter}{4}

\newcommand{\ResHDOracleLL}{0.437}

\newcommand{\ResHDPacflLL}{0.451}

\newcommand{\ResHDPtopoAUC}{0.879}
\newcommand{\ResHDPtopoLL}{0.438}

\newcommand{\ResHDRandomLL}{0.454}
\newcommand{\ResHDRhoCmp}{not below}
\newcommand{\ResHDRhoOrac}{0.56}
\newcommand{\ResHDRhoRand}{0.61}
\newcommand{\ResHDRhoTopo}{0.61}

\newcommand{\RobFPR}{0.12}
\newcommand{\RobFracMax}{40}
\newcommand{\RobFracMid}{20}
\newcommand{\RobGeDitto}{0.403}

\newcommand{\RobGeFedavg}{0.472}

\newcommand{\RobGeIfca}{0.414}

\newcommand{\RobGeMedian}{0.465}

\newcommand{\RobGePtopo}{0.403}

\newcommand{\RobLfDittoMax}{0.411}

\newcommand{\RobLfFedavgMax}{0.594}

\newcommand{\RobLfPtopoMax}{0.605}

\newcommand{\RobPoPtopo}{0.376}

\newcommand{\RobPoPtopoNT}{0.373}

\newcommand{\RobTPR}{0.13}
\newcommand{\ThAlphaCalTen}{0.023}
\newcommand{\ThAlphaEffVone}{0.89}

\newcommand{\ThBiasTauOne}{1e-10}
\newcommand{\ThBiasTightness}{325}
\newcommand{\ThContractionWorst}{0.948}
\newcommand{\ThKappa}{177}
\newcommand{\ThQten}{0.945}
}{}
\providecommand{\AblAlphaOne}{\textcolor{red}{??}}
\providecommand{\AblAlphaZero}{\textcolor{red}{??}}
\providecommand{\AblFull}{\textcolor{red}{??}}
\providecommand{\AblNoClust}{\textcolor{red}{??}}
\providecommand{\AblNoTrust}{\textcolor{red}{??}}
\providecommand{\AblUnif}{\textcolor{red}{??}}
\providecommand{\AblVoneAlpha}{\textcolor{red}{??}}
\providecommand{\AlignAlARI}{\textcolor{red}{??}}
\providecommand{\AlignAlARIsub}{\textcolor{red}{??}}
\providecommand{\AlignCoARI}{\textcolor{red}{??}}
\providecommand{\AlignCoDitto}{\textcolor{red}{??}}
\providecommand{\AlignCoFedavg}{\textcolor{red}{??}}
\providecommand{\AlignCoIfca}{\textcolor{red}{??}}
\providecommand{\AlignCoPtopo}{\textcolor{red}{??}}
\providecommand{\AlignCoRhoOrac}{\textcolor{red}{??}}
\providecommand{\AlignCoRhoTopo}{\textcolor{red}{??}}
\providecommand{\AlignCorr}{\textcolor{red}{??}}
\providecommand{\AlignGeARI}{\textcolor{red}{??}}
\providecommand{\AlignGeGain}{\textcolor{red}{??}}
\providecommand{\AlignGeRhoOrac}{\textcolor{red}{??}}
\providecommand{\AlignTaARI}{\textcolor{red}{??}}
\providecommand{\AlignTaARIsub}{\textcolor{red}{??}}
\providecommand{\AlignTaFedavg}{\textcolor{red}{??}}
\providecommand{\AlignTaIfca}{\textcolor{red}{??}}
\providecommand{\AlignTaPacfl}{\textcolor{red}{??}}
\providecommand{\AlignTaPtopo}{\textcolor{red}{??}}
\providecommand{\ContAUC}{\textcolor{red}{??}}
\providecommand{\ContFPR}{\textcolor{red}{??}}
\providecommand{\ContTPR}{\textcolor{red}{??}}
\providecommand{\CostDescOne}{\textcolor{red}{??}}
\providecommand{\CostHoneNumpy}{\textcolor{red}{??}}
\providecommand{\DPFedavgLLinf}{\textcolor{red}{??}}
\providecommand{\DPFedavgLLmid}{\textcolor{red}{??}}
\providecommand{\DPFedavgMIAinf}{\textcolor{red}{??}}
\providecommand{\DPLocalLL}{\textcolor{red}{??}}
\providecommand{\DPPtopoARImid}{\textcolor{red}{??}}
\providecommand{\DPPtopoLLinf}{\textcolor{red}{??}}
\providecommand{\DPPtopoLLmid}{\textcolor{red}{??}}
\providecommand{\DPPtopoMIAinf}{\textcolor{red}{??}}
\providecommand{\DPPtopoNPLLmid}{\textcolor{red}{??}}
\providecommand{\DPclip}{\textcolor{red}{??}}
\providecommand{\DPepsMid}{\textcolor{red}{??}}
\providecommand{\DPepsOK}{\textcolor{red}{??}}
\providecommand{\DPgapPos}{\textcolor{red}{??}}
\providecommand{\DPinterMed}{\textcolor{red}{??}}
\providecommand{\DPlooMax}{\textcolor{red}{??}}
\providecommand{\DPlooMed}{\textcolor{red}{??}}
\providecommand{\DPsafeSigma}{\textcolor{red}{??}}
\providecommand{\DPsigPhiMid}{\textcolor{red}{??}}
\providecommand{\DPsigmaOK}{\textcolor{red}{??}}
\providecommand{\DPsteps}{\textcolor{red}{??}}
\providecommand{\HDsizes}{\textcolor{red}{??}}
\providecommand{\NumSeeds}{\textcolor{red}{??}}
\providecommand{\PCAWorseFedavg}{\textcolor{red}{??}}
\providecommand{\PCBWorseFedavg}{\textcolor{red}{??}}
\providecommand{\PCHDWorseFedavg}{\textcolor{red}{??}}
\providecommand{\PCHDWorseLocal}{\textcolor{red}{??}}
\providecommand{\PropInfClean}{\textcolor{red}{??}}
\providecommand{\PropInfMid}{\textcolor{red}{??}}
\providecommand{\ResAARI}{\textcolor{red}{??}}
\providecommand{\ResADittoLL}{\textcolor{red}{??}}
\providecommand{\ResAFedavgLL}{\textcolor{red}{??}}
\providecommand{\ResAFedproxLL}{\textcolor{red}{??}}
\providecommand{\ResAIfcaLL}{\textcolor{red}{??}}
\providecommand{\ResALocalLL}{\textcolor{red}{??}}
\providecommand{\ResANBetter}{\textcolor{red}{??}}
\providecommand{\ResAOracleLL}{\textcolor{red}{??}}
\providecommand{\ResAPacflLL}{\textcolor{red}{??}}
\providecommand{\ResAPfedmeLL}{\textcolor{red}{??}}
\providecommand{\ResAPtopoAUC}{\textcolor{red}{??}}
\providecommand{\ResAPtopoLL}{\textcolor{red}{??}}
\providecommand{\ResARandomLL}{\textcolor{red}{??}}
\providecommand{\ResARhoOrac}{\textcolor{red}{??}}
\providecommand{\ResARhoRand}{\textcolor{red}{??}}
\providecommand{\ResARhoTopo}{\textcolor{red}{??}}
\providecommand{\ResAScaffoldLL}{\textcolor{red}{??}}
\providecommand{\ResBDittoLL}{\textcolor{red}{??}}
\providecommand{\ResBFedavgLL}{\textcolor{red}{??}}
\providecommand{\ResBIfcaLL}{\textcolor{red}{??}}
\providecommand{\ResBPtopoLL}{\textcolor{red}{??}}
\providecommand{\ResBRhoRand}{\textcolor{red}{??}}
\providecommand{\ResBRhoTopo}{\textcolor{red}{??}}
\providecommand{\ResHDBest}{\textcolor{red}{??}}
\providecommand{\ResHDDittoLL}{\textcolor{red}{??}}
\providecommand{\ResHDFedavgAUC}{\textcolor{red}{??}}
\providecommand{\ResHDFedavgLL}{\textcolor{red}{??}}
\providecommand{\ResHDNBetter}{\textcolor{red}{??}}
\providecommand{\ResHDOracleLL}{\textcolor{red}{??}}
\providecommand{\ResHDPacflLL}{\textcolor{red}{??}}
\providecommand{\ResHDPtopoAUC}{\textcolor{red}{??}}
\providecommand{\ResHDPtopoLL}{\textcolor{red}{??}}
\providecommand{\ResHDRandomLL}{\textcolor{red}{??}}
\providecommand{\ResHDRhoCmp}{\textcolor{red}{??}}
\providecommand{\ResHDRhoOrac}{\textcolor{red}{??}}
\providecommand{\ResHDRhoRand}{\textcolor{red}{??}}
\providecommand{\ResHDRhoTopo}{\textcolor{red}{??}}
\providecommand{\RobFPR}{\textcolor{red}{??}}
\providecommand{\RobFracMax}{\textcolor{red}{??}}
\providecommand{\RobFracMid}{\textcolor{red}{??}}
\providecommand{\RobGeDitto}{\textcolor{red}{??}}
\providecommand{\RobGeFedavg}{\textcolor{red}{??}}
\providecommand{\RobGeIfca}{\textcolor{red}{??}}
\providecommand{\RobGeMedian}{\textcolor{red}{??}}
\providecommand{\RobGePtopo}{\textcolor{red}{??}}
\providecommand{\RobLfDittoMax}{\textcolor{red}{??}}
\providecommand{\RobLfFedavgMax}{\textcolor{red}{??}}
\providecommand{\RobLfPtopoMax}{\textcolor{red}{??}}
\providecommand{\RobPoPtopo}{\textcolor{red}{??}}
\providecommand{\RobPoPtopoNT}{\textcolor{red}{??}}
\providecommand{\RobTPR}{\textcolor{red}{??}}
\providecommand{\ThAlphaCalTen}{\textcolor{red}{??}}
\providecommand{\ThAlphaEffVone}{\textcolor{red}{??}}
\providecommand{\ThBiasTauOne}{\textcolor{red}{??}}
\providecommand{\ThBiasTightness}{\textcolor{red}{??}}
\providecommand{\ThContractionWorst}{\textcolor{red}{??}}
\providecommand{\ThKappa}{\textcolor{red}{??}}
\providecommand{\ThQten}{\textcolor{red}{??}}
\newcommand{\restable}[1]{\IfFileExists{\resdir/tables/#1.tex}%
  {\input{\resdir/tables/#1.tex}}%
  {\fbox{\parbox{0.8\linewidth}{\centering\small Table \texttt{#1} not found: run \texttt{ptopofl\_v2.py}.}}}}
\newcommand{\resfig}[2]{\IfFileExists{\resdir/figures/#1.pdf}%
  {\includegraphics[width=#2]{\resdir/figures/#1.pdf}}%
  {\fbox{\parbox{0.8\linewidth}{\centering\small Figure \texttt{#1} not found: run \texttt{ptopofl\_v2.py}.}}}}
\newcommand{\codeavailability}{Code: \url{https://github.com/MorillaLab/TopoFederatedL}.}
\newif\ifversionnote \versionnotetrue

\title{\ptopofl: Topology-Guided Personalised Federated\\
       Learning with Persistent-Homology Client Descriptors}

\author{\name Kelly L Vomo-Donfack \email vomodonfack@math.univ-paris13.fr \\
      \addr Universit\'e Sorbonne Paris Nord, LAGA, CNRS, UMR 7539,\\
      Laboratoire d'excellence Infibrex, F-93430 Villetaneuse, France\\
      Instituto de Hortofruticultura Subtropical y Mediterr\'anea La Mayora
      (IHSM), Universidad de M\'alaga--Consejo Superior de Investigaciones
      Cient\'ificas, M\'alaga, Spain
      \AND
      \name Adryel Hoszu \email adryel.hoszu@edu.univ-paris13.fr \\
      \addr Universit\'e Sorbonne Paris Nord, LAGA, CNRS, UMR 7539,\\
      Laboratoire d'excellence Infibrex, F-93430 Villetaneuse, France\\
      Instituto de Hortofruticultura Subtropical y Mediterr\'anea La Mayora
      (IHSM), Universidad de M\'alaga--Consejo Superior de Investigaciones
      Cient\'ificas, M\'alaga, Spain
      \AND
      \name Gr\'egory Ginot \email ginot@math.univ-paris13.fr \\
      \addr Universit\'e Sorbonne Paris Nord, LAGA, CNRS, UMR 7539,\\
      Laboratoire d'excellence Infibrex, F-93430 Villetaneuse, France
      \AND
      \name Ian Morilla \email morilla@math.univ-paris13.fr \\
      \addr Universit\'e Sorbonne Paris Nord, LAGA, CNRS, UMR 7539,\\
      Laboratoire d'excellence Infibrex, F-93430 Villetaneuse, France\\
      Instituto de Hortofruticultura Subtropical y Mediterr\'anea La Mayora
      (IHSM), Universidad de M\'alaga--Consejo Superior de Investigaciones
      Cient\'ificas, M\'alaga, Spain}

\begin{document}

\maketitle

\begin{abstract}
Clustered and personalised federated learning (FL) needs a signal that tells the server
which clients should share a model. Most methods derive it from losses or model updates,
which must be recomputed as training proceeds. We study \ptopofl, which derives it once,
before training, from a compact persistent-homology (PH) descriptor of each client's
feature distribution: the server clusters clients on these descriptors, weights them within
clusters, and runs a blended clustered version of local gradient descent that returns one
personalised model per cluster. We analyse this complete procedure---clustering, weighting
and blending included---for strongly convex and smooth local losses. For any partition, one
round is a contraction with factor $(1-\eta\mu)^{\tau}$, where $\eta$ is the step size and
$\tau$ the number of local steps. We characterise its fixed point explicitly through a
local-step drift bound, an exact formula for the effective strength of blending (which
yields a calibration rule), and a stochastic floor that exposes a bias--variance trade-off
of clustering. We also prove that personalisation lowers the heterogeneity floor in
proportion to the share of the dispersion of client optima that the partition captures.
This share is measurable, which replaces an unverifiable ``topology--gradient alignment''
assumption by an empirical question. Experiments on controlled synthetic federations, a
label-skew benchmark and a real four-hospital dataset confirm the theory and delimit the
method. When feature geometry and decision rules co-vary, the PH partition recovers the
latent client groups, including geometric differences that a linear-subspace descriptor
misses, and \ptopofl\ is competitive with the strongest clustered baselines. When
heterogeneity is invisible in feature geometry, it offers no advantage over FedAvg. On the real four-hospital federation, \ptopofl\ lowers
FedAvg's test log-loss from $\ResHDFedavgLL$ to $\ResHDPtopoLL$, on par with the best
personalised baselines. We
also state what \ptopofl\ does not provide. It transmits model parameters, so it is no
more private than FedAvg; we give a record-level $(\varepsilon,\delta)$ differential
privacy accounting for everything a client transmits and show, with attack-based
evaluations, that privatising the descriptor under its worst-case sensitivity destroys the
partition. Code that regenerates every number in the paper is provided.
\end{abstract}

\ifversionnote
\begin{center}\begin{minipage}{0.92\linewidth}\small
\textbf{Note on this version.} This version replaces arXiv:2603.04323v1. The numerical results
reported in v1 were not produced by running the released code and are withdrawn. Every number,
table and figure in this version is generated by the released script (Appendix~\ref{app:repro}).
\end{minipage}\end{center}
\fi


\section{Introduction}
\label{sec:intro}

Federated learning (FL) \citep{mcmahan2017communication,kairouz2021advances} trains models
over $K$ clients holding private datasets $\calD_k$ without centralising the data, typically
by solving
\begin{equation}
  \min_{\theta\in\R^p} F(\theta):=\sum_{k=1}^K p_k F_k(\theta),
  \qquad F_k(\theta):=\frac{1}{n_k}\sum_{(x,y)\in\calD_k}\ell(\theta;x,y),
  \qquad p_k=\frac{n_k}{\sum_j n_j}.
  \label{eq:global_obj}
\end{equation}
When client distributions differ, a single model can fit every client poorly. A natural
remedy is to let \emph{groups} of similar clients share a model, as in clustered and
personalised FL \citep{ghosh2020efficient,sattler2020clustered,long2023multi,t2020personalized,li2021ditto}.
Clustered methods need a similarity signal. Most derive it from the optimisation itself:
the losses of candidate models \citep[IFCA;][]{ghosh2020efficient}, the geometry of
updates \citep[CFL;][]{sattler2020clustered,briggs2020federated}, or distances between local
models \citep[FeSEM;][]{long2023multi}. Such signals are entangled with the current model
and must be recomputed during training. A second line of work uses a \emph{one-shot}
summary of each client's data, such as local $k$-means centroids
\citep{dennis2021heterogeneity} or the principal subspace of the local data matrix
\citep[PACFL;][]{vahidian2023efficient}.

This paper studies a third kind of summary: persistent homology (PH). In \ptopofl, each
client computes once a $46$-dimensional descriptor of the multiscale shape of its feature
cloud (Betti curves, persistence entropy and amplitude in homological degrees $0$ and $1$).
The server clusters clients on these descriptors, weights them within clusters, and then
runs a blended, clustered version of local gradient descent that broadcasts one
personalised model per cluster. Because the descriptor is computed after a per-coordinate
min--max normalisation, it ignores location and scale and responds to properties of shape
that first- and second-moment summaries miss, such as intrinsic dimension and tail
behaviour (Section~\ref{sec:e_alignment}). Whether such a geometric signal helps a learning
task depends on whether clients that look alike in feature space also prefer similar
models. The earlier version of this work assumed this ``topology--gradient alignment''; here
we make the dependence explicit and \emph{measurable} instead.

\paragraph{Why persistent homology.}
PH summarises a point cloud by the scales at which connected components and loops appear and
disappear. In centralised machine learning it has been used for graph classification
\citep{hofer2017deep}, time series \citep{umeda2017time}, decision-boundary regularisation
\citep{chen2019topological} and medical imaging \citep{clough2020topological,abaach2023geotop,gouiaa2024taelcore};
there, persistence diagrams are turned into fixed-length inputs by vectorisations such as
Betti curves, persistence entropy, landscapes, silhouettes or persistence images
\citep{giusti2015clique,chintakunta2015entropy,chazal2021introduction,bubenik2015statistical,chazal2014stochastic,adams2017persistence,reininghaus2015stable}.
Our descriptor is one such vectorisation, chosen for its low dimension and low cost; we
compare clients with Euclidean distances between descriptor vectors and do \emph{not} use
the Wasserstein geometry of the space of diagrams \citep{mileyko2011probability,turner2014frechet}.
Landscapes or images could replace it without changing any result of
Section~\ref{sec:theory}, which hold for \emph{any} partition and weights.

\paragraph{Scope with respect to privacy.}
\ptopofl\ transmits model parameters $\theta_k$ in every round, exactly like FedAvg. Every
attack that applies to FedAvg messages therefore applies to \ptopofl; in particular, the
difference of two consecutive models of a client reveals its update. \ptopofl\ is not a
privacy mechanism, and its descriptor is an \emph{additional} release. What we do provide is
(i) a protocol in which the server touches models only through per-cluster weighted sums,
so that secure aggregation applies, and (ii) a record-level $(\varepsilon,\delta)$
differential-privacy accounting for \emph{everything} a client transmits, together with its
utility cost and attack-based evaluations (Sections~\ref{sec:comm} and~\ref{sec:e_privacy}).

\paragraph{Contributions.}
\begin{enumerate}[leftmargin=1.4em,itemsep=2pt]
  \item \textbf{A precisely specified method} (Algorithm~\ref{alg:ptopofl},
    Figure~\ref{fig:framework}), with every transmitted object stated.
  \item \textbf{A convergence theory for the algorithm that is actually run}
    (Section~\ref{sec:theory}; complete proofs in Appendix~\ref{app:proofs}).
    \begin{itemize}[leftmargin=1.2em,itemsep=1pt]
      \item For every partition, weighting and blending coefficient, one round is a
        $(1-\eta\mu)^{\tau}$-contraction (Theorem~\ref{thm:contraction}).
      \item The fixed point is characterised by a local-step drift bound
        (Theorem~\ref{thm:drift}) and by an exact formula for the effective strength of
        blending, which yields a calibration rule (Proposition~\ref{prop:blend}).
      \item Personalisation provably lowers the heterogeneity floor in proportion to a
        measurable share $1-\rho(\calC)$ of the dispersion of client optima
        (Proposition~\ref{prop:anova}, Corollary~\ref{cor:endtoend}).
      \item With stochastic local steps there is a bias--variance trade-off of clustering
        (Theorem~\ref{thm:stochastic}).
      \item We also give a correct partition-stability result (Theorem~\ref{thm:partition})
        and bounds on trust weights (Propositions~\ref{prop:trust} and~\ref{prop:cap}).
    \end{itemize}
    No assumption linking topology to gradients is needed: topology enters only through
    the partition, whose quality we measure.
  \item \textbf{Experiments that test the theory and delimit the method}
    (Section~\ref{sec:experiments}).
    \begin{itemize}[leftmargin=1.2em,itemsep=1pt]
      \item Settings: synthetic federations with controlled alignment between geometry and
        decision rules, a label-skew benchmark, and a real four-hospital dataset
        \citep{detrano1989international,terrail2022flamby}.
      \item Comparisons: eight baselines, including a subspace-descriptor method in the
        spirit of PACFL, plus random and oracle partitions.
      \item Additional studies: corrupted clients, and a differential-privacy sweep with
        membership- and property-inference attacks.
    \end{itemize}
  \item \textbf{An explicit account of limitations}: what the descriptor sees and does not
    see, the weakness of descriptor-based anomaly detection, and the privacy cost of
    privatising the descriptor.
\end{enumerate}

\section{Background}
\label{sec:background}

\subsection{Federated optimisation under heterogeneity}
\label{sec:bg_fl}
In FedAvg \citep{mcmahan2017communication} the server broadcasts a model, each client
performs $\tau$ local (stochastic) gradient steps on $F_k$, and the server averages the
results with weights $p_k$. With identical client distributions local steps are
essentially free \citep{stich2019local,woodworth2020minibatch}. Under heterogeneity,
however, the iteration converges (with constant step size) to a fixed point that differs from
the minimiser of~\eqref{eq:global_obj}: local steps introduce a \emph{drift}
\citep{zhao2018federated,malinovsky2020local,charles2021convergence,pathak2020fedsplit}.
Analyses of local SGD quantify this drift through a heterogeneity constant at the optimum,
such as $\zeta^2=\sum_k p_k\norm{\nabla F_k(\theta^\star)}^2$
\citep{khaled2020tighter,koloskova2020unified,woodworth2020minibatch}. Proximal terms
\citep[FedProx;][]{li2020federated} and control variates
\citep[SCAFFOLD;][]{karimireddy2020scaffold} reduce the drift of a single model.
Personalised methods
\citep[pFedMe, Ditto, APFL;][]{t2020personalized,li2021ditto,deng2020adaptive,hanzely2020federated}
and clustered methods \citep{ghosh2020efficient,sattler2020clustered,long2023multi} instead
change the target, so that different clients may use different models. Our analysis
(Section~\ref{sec:theory}) adopts the fixed-point viewpoint of
\citet{malinovsky2020local} and \citet{charles2021convergence} and extends it to clustered,
weighted and blended aggregation.

\subsection{Persistent homology}
\label{sec:bg_ph}
For a finite point cloud $X\subset\R^d$, the Vietoris--Rips filtration contains, at scale
$\epsilon$, every simplex whose vertices are pairwise within distance $\epsilon$. Each
homological feature (a connected component in $H_0$, a loop in $H_1$) has a birth scale $b$
and a death scale $d$; the multiset of pairs $(b,d)$ is the persistence diagram
$\mathrm{Dgm}_j(X)$ \citep{edelsbrunner2010computational}. Diagrams endowed with the
bottleneck or $p$-Wasserstein distances form metric spaces
\citep{mileyko2011probability,turner2014frechet}, and Rips diagrams are stable: with the
diameter convention above,
\begin{equation}
  d_B\bigl(\mathrm{Dgm}_j(X),\mathrm{Dgm}_j(Y)\bigr)\le 2\,d_{\mathrm{GH}}(X,Y)\le 2\,d_H(X,Y)
  \label{eq:stability}
\end{equation}
\citep{chazal2009gromov,chazal2014persistence}. We stress a limitation that matters below:
\eqref{eq:stability} holds at the level of diagrams. It does not transfer to many standard
vectorisations, because counts and Betti curves evaluated at fixed thresholds are
discontinuous functions of the data, and persistence entropy is not uniformly Lipschitz
\citep{atienza2020stability}. For this reason all our statements about the partition are
expressed in terms of the computed descriptor vectors, never in terms of perturbations of
the data.

\subsection{Privacy in FL}
\label{sec:bg_privacy}
Shared gradients and model updates can be inverted to reconstruct training data
\citep{zhu2019deep,geiping2020inverting}, and models leak membership \citep{yeom2018privacy}
and distributional properties of the training data \citep{melis2019exploiting}.
Differential privacy (DP) \citep{dwork2014algorithmic} bounds such leakage for every
adversary; DP-SGD \citep{abadi2016deep,wei2020federated} achieves it by clipping and
noising gradients. Secure aggregation \citep[SecAgg;][]{bonawitz2017practical} hides
individual updates from the server but reveals their sum. For topological summaries, the
sensitivity of \v{C}ech persistence diagrams to a single record does not decrease with the
sample size, which makes record-level DP for standard PH summaries costly \citep{kang2024differentially}; our
experiments confirm this for the descriptor used here (Section~\ref{sec:e_privacy}).

\section{The \ptopofl\ framework}
\label{sec:framework}

\begin{figure}[t]
\centering
\begin{tikzpicture}[font=\small, >={Stealth[length=2mm]},
  cl/.style={draw=blue!60!black, rounded corners=2pt, fill=blue!5, minimum width=2.25cm,
             minimum height=1.0cm, align=center},
  sv/.style={draw=orange!70!black, rounded corners=2pt, fill=orange!8, minimum width=5.3cm,
             minimum height=0.95cm, align=center},
  lab/.style={font=\footnotesize, fill=white, inner sep=1pt}]
  \node[cl] (c1) at (0,2.3)  {Client $1$\\[-1pt]\footnotesize local data $\calD_1$};
  \node[cl] (c2) at (0,0.65) {Client $2$\\[-1pt]\footnotesize local data $\calD_2$};
  \node     (cd) at (0,-0.35) {$\vdots$};
  \node[cl] (cK) at (0,-1.35) {Client $K$\\[-1pt]\footnotesize local data $\calD_K$};
  \node[sv] (s0) at (8.0,2.55) {\textbf{Round 0 (once):} partition $\calC$, trust $t_k$,\\
                               weights $w_k$, cluster masses $\pi_j$};
  \node[sv] (s1) at (8.0,0.85) {\textbf{Every round:} per-cluster sums\\
                               $u_j=\sum_{k\in C_j}w_k\theta_k^{t+1}$ \;(SecAgg-compatible)};
  \node[sv] (s2) at (8.0,-0.85) {\textbf{Blending:} $\theta_j^{t+1}=(1-\alpha)\,u_j+\alpha\sum_i\pi_i u_i$};
  \begin{scope}[on background layer]
    \node[draw=orange!70!black, dashed, rounded corners=3pt, fit=(s0)(s1)(s2), inner sep=5pt,
          label={[font=\small\bfseries]above:Server}] {};
  \end{scope}
  \draw[->, thick, green!45!black] (c1.east) -- node[lab, pos=0.42]{$\phi_1\in\R^{46}$ (once)} (s0.west);
  \draw[->, thick, blue!70!black]  ([yshift=-3pt]c1.south east) -- node[lab, pos=0.38]{$\theta_1^{t+1}$ ($p$ floats)} (s1.west);
  \draw[->, thick, blue!70!black, dashed] (s2.west) -- node[lab, pos=0.62]{$\theta_{c(1)}^{t+1}$} ([yshift=3pt]c1.south east);
  \draw[->, thick, green!45!black] (cK.east) -- (s0.south west);
  \draw[->, thick, blue!70!black]  ([yshift=6pt]cK.east) -- (s1.south west);
  \draw[->, thick, blue!70!black, dashed] (s2.south west) -- node[lab, pos=0.55]{$\theta_{c(K)}^{t+1}$} (cK.south east);
  \draw[->] (s0) -- (s1);
  \draw[->] (s1) -- (s2);
\end{tikzpicture}
\caption{\textbf{\ptopofl\ protocol.} Every client sends two kinds of message. Green: the
descriptor $\phi_k$, sent \emph{once}, before training. Solid blue: its updated local model
$\theta_k^{t+1}\in\R^p$, sent in every round. Dashed blue: the server returns to client $k$
the personalised model of its cluster, $\theta_{c(k)}^{t+1}$. The server uses descriptors
only to build the partition and the weights, and uses models only through the per-cluster
weighted sums $u_j$. Arrows are drawn for clients $1$ and $K$; all clients behave alike.}
\label{fig:framework}
\end{figure}

\paragraph{Protocol and threat model.}
Training has two phases (Figure~\ref{fig:framework}, Algorithm~\ref{alg:ptopofl}).
\begin{itemize}[leftmargin=1.2em,itemsep=2pt]
  \item \emph{Descriptor phase (round $0$, once).} Each client computes and sends a
    descriptor $\phi_k$ (Section~\ref{sec:descriptor}). The server computes a partition
    $\calC=\{C_1,\dots,C_M\}$, trust scores $t_k$, within-cluster weights $w_k$ and cluster
    masses $\pi_j$ (Section~\ref{sec:structure}). It then tells each client its cluster
    index $c(k)$.
  \item \emph{Training rounds} ($t=0,\dots,T-1$). The server holds one model per cluster,
    $\Theta^t=(\theta_1^t,\dots,\theta_M^t)$. Client $k$ receives $\theta^t_{c(k)}$, runs
    $\tau$ local gradient steps on $F_k$, and returns $\theta_k^{t+1}$. The server
    aggregates within clusters and blends across clusters (Section~\ref{sec:training}).
\end{itemize}
We consider an honest-but-curious server and passive eavesdroppers, which observe all
messages, and clients that follow the protocol. Some clients may hold corrupted data
(Section~\ref{sec:theory_trust}). Descriptors are self-reported, so a client that deviates
adaptively from the protocol can forge an honest-looking $\phi_k$. Defending against such
Byzantine clients is outside our scope.

\subsection{Topological client descriptor}
\label{sec:descriptor}

Client $k$ draws $n_{\mathrm{rep}}=5$ subsamples of $n_{\mathrm{sub}}=80$ feature vectors
from its (standardised) training features, using a client-specific seed. It min--max
normalises each coordinate within the subsample and computes Vietoris--Rips persistence: $H_0$
exactly, by Kruskal's algorithm, and $H_1$ with Ripser \citep{ripser} on the filtration
truncated at the $95$th percentile of pairwise distances. From each diagram
$\mathrm{Dgm}_j$, $j\in\{0,1\}$, it extracts four summaries:
\begin{itemize}[leftmargin=1.2em,itemsep=2pt]
  \item the Betti curve $\mathrm{BC}_j\in\R^{L}$: the number of features alive at $L=20$
    thresholds, linearly spaced between $0$ and the $95$th percentile of finite death
    values \citep{giusti2015clique,chazal2021introduction};
  \item the number $N_j$ of finite pairs with positive persistence;
  \item the persistence entropy \citep{chintakunta2015entropy} $E_j=-\sum_i r_i\log r_i$, where
    $r_i=\mathrm{pers}_i/\sum_l\mathrm{pers}_l$;
  \item the amplitude $A_j=(\sum_i\mathrm{pers}_i^2)^{1/2}$.
\end{itemize}
The descriptor is the average over the $n_{\mathrm{rep}}$ subsamples of
\begin{equation}
  \bigl[\mathrm{BC}_0,\ \mathrm{BC}_1,\ N_0,\ N_1,\ E_0,\ E_1,\ A_0,\ A_1\bigr]\in\R^{2L+6}=\R^{46},
  \label{eq:descriptor}
\end{equation}
and the server works with $\hat\phi_k=\phi_k/\norm{\phi_k}$. (An earlier version of the
implementation stored $N_j$ twice, under two names; the duplicate is removed.)

\paragraph{Properties.}
\begin{itemize}[leftmargin=1.2em,itemsep=2pt]
  \item \emph{Invariances.} The descriptor is invariant to permutations of the records and,
    because of the per-coordinate normalisation, to coordinate-wise affine maps
    $x\mapsto Dx+b$ with $D$ diagonal and positive. It is \emph{not} rotation invariant.
  \item \emph{What it sees.} Averaging over subsamples makes it a stable estimate of a
    population quantity: with $n_{\mathrm{rep}}=1$, partition recovery degrades
    (Section~\ref{sec:e_ablation}). Empirically (Section~\ref{sec:e_alignment}) it separates
    clients that differ in intrinsic dimension or tail heaviness, but not clients that
    differ only by mean shifts, which the normalisation removes, or only in their labelling
    rule, since labels are never used.
  \item \emph{No data-level stability.} As noted in Section~\ref{sec:bg_ph}, the map from
    data to descriptor is not Lipschitz, so we make no claim of that kind.
  \item \emph{Cost.} The $H_0$ computation sorts all pairs, at cost
    $O(n_{\mathrm{sub}}^2\log n_{\mathrm{sub}})$. The $H_1$ cost is polynomial in the
    number of simplices; the measured times are in Appendix~\ref{app:cost}.
\end{itemize}

\subsection{Server-side structure: partition, trust and weights}
\label{sec:structure}

All quantities in this subsection are computed once, from $\{\hat\phi_k\}$ and the sample
sizes $\{n_k\}$ only.

\paragraph{Partition.}
Average-linkage agglomerative clustering of $\{\hat\phi_k\}$ in Euclidean distance, cut at
$M$ clusters, gives $\calC=\{C_1,\dots,C_M\}$; we write $c(k)$ for the cluster of client $k$.

\paragraph{Trust.}
Let $\delta_k=\frac{1}{K-1}\sum_{l\neq k}\norm{\hat\phi_k-\hat\phi_l}$. We use the robust
score
\[
  z_k=\frac{\delta_k-\med(\delta)}{\max\{1.4826\,\MAD(\delta),\,s_0\med(\delta)\}},
  \qquad t_k=\exp\bigl(-\max\{z_k-z_0,\,0\}\bigr),
\]
with $z_0=2$ and $s_0=0.05$ \citep{huber2009robust}. An earlier version used the mean and
standard deviation of $\delta$; Proposition~\ref{prop:cap} shows that this caps the
achievable down-weighting at a level fixed by the fraction of corrupted clients alone.

\paragraph{Weights and masses.}
For $k\in C_j$, with $\bar\phi_j$ the mean of $\{\hat\phi_l\}_{l\in C_j}$,
\begin{equation}
  w_k=\frac{n_k\,e^{-\norm{\hat\phi_k-\bar\phi_j}}\,t_k}{\sum_{l\in C_j}n_l\,e^{-\norm{\hat\phi_l-\bar\phi_j}}\,t_l},
  \qquad
  \pi_j=\frac{\sum_{k\in C_j}n_k t_k}{\sum_{k}n_k t_k}.
  \label{eq:weights}
\end{equation}
Thus $\sum_{k\in C_j}w_k=1$ for every $j$, $\sum_j\pi_j=1$, and $\omega_k:=\pi_{c(k)}w_k$
are mixture weights over all clients.

\subsection{Training: blended clustered local gradient descent}
\label{sec:training}

\paragraph{Local objective.}
In this paper $F_k$ is the $\ell_2$-regularised logistic loss with all parameters penalised,
$F_k(\theta)=\frac1{n_k}\sum_i\log(1+e^{-\tilde y_i\langle\theta,\tilde x_i\rangle})+\frac\lambda2\norm{\theta}^2$,
where $\tilde x_i=(x_i,1)$ and $\tilde y_i\in\{\pm1\}$. It is $\mu$-strongly convex with
$\mu=\lambda$, and $L_k$-smooth with $L_k=\lambda+\norm{\tilde X_k}_2^2/(4n_k)$, so the
assumptions of Section~\ref{sec:theory} hold with explicit constants.

\paragraph{Round.}
Client $k\in C_j$ starts from $\theta^t_j$ and performs $\tau$ gradient steps with step size
$\eta\le1/L$, where $L=\max_k L_k$:
\[
  \theta_k^{t+1}=T_k(\theta_j^t),\qquad T_k:=(\Id-\eta\nabla F_k)^{\tau}.
\]
The server then forms
\begin{equation}
  u_j=\sum_{k\in C_j}w_k\,\theta_k^{t+1},
  \qquad
  \theta_j^{t+1}=(1-\alpha)\,u_j+\alpha\sum_{i=1}^M\pi_i\,u_i .
  \label{eq:round}
\end{equation}
Three special cases are worth noting:
\begin{itemize}[leftmargin=1.2em,itemsep=2pt]
  \item $M=1$ is FedAvg with weights $w$;
  \item $\alpha=1$ trains one consensus model with weights $\omega$;
  \item $\alpha=0$ runs $M$ independent weighted FedAvg instances.
\end{itemize}
Each client finally uses $\theta^T_{c(k)}$.

\paragraph{Calibrating the blending coefficient.}
Proposition~\ref{prop:blend} shows that, at the fixed point, the fraction of consensus
induced by $\alpha$ is $\alpha/(1-(1-\alpha)q)\ge\alpha$, where $q=(1-\eta\mu)^{\tau}$. With
few local steps relative to the conditioning, even a small $\alpha$ removes most of the
personalisation. For example, with our default hyper-parameters $q=\ThQten$, so the value
$\alpha=0.3$ used in an earlier version corresponds to an effective consensus fraction of
$\ThAlphaEffVone$. We therefore specify a target effective fraction $\bar\alpha$ and set
\begin{equation}
  \alpha=\frac{\bar\alpha\,(1-q)}{1-\bar\alpha\,q},\qquad q=(1-\eta\lambda)^{\tau}.
  \label{eq:alphacal}
\end{equation}
We use $\bar\alpha=0.3$ throughout, without tuning; this gives $\alpha=\ThAlphaCalTen$.

\begin{algorithm}[t]
\caption{\ptopofl}
\label{alg:ptopofl}
\begin{algorithmic}[1]
\Require clients $k=1,\dots,K$; number of clusters $M$; target blending $\bar\alpha$;
         local steps $\tau$; step size $\eta\le 1/L$; rounds $T$
\State \textbf{Descriptor phase (once).} Each client sends $\phi_k$ (Eq.~\ref{eq:descriptor}), optionally with Gaussian noise (Prop.~\ref{prop:dp})
\State Server: $\hat\phi_k\gets\phi_k/\norm{\phi_k}$; $\calC\gets$ average-linkage clustering of $\{\hat\phi_k\}$ into $M$ clusters
\State Server: compute $t_k$, $w_k$, $\pi_j$ (Section~\ref{sec:structure}); $\alpha$ from Eq.~\eqref{eq:alphacal}; send $c(k)$ (and $w_k$, for SecAgg) to client $k$
\State $\theta_j^0\gets 0$ for $j=1,\dots,M$
\For{$t=0,\dots,T-1$}
  \For{each client $k$ \textbf{in parallel}}
     \State $\theta\gets\theta^t_{c(k)}$; \textbf{repeat} $\tau$ times: $\theta\gets\theta-\eta\nabla F_k(\theta)$; \ \textbf{send} $\theta_k^{t+1}\gets\theta$
  \EndFor
  \State Server: $u_j\gets\sum_{k\in C_j}w_k\theta_k^{t+1}$ for all $j$ \Comment{per-cluster SecAgg possible}
  \State Server: $\theta_j^{t+1}\gets(1-\alpha)u_j+\alpha\sum_i\pi_iu_i$; send $\theta^{t+1}_{c(k)}$ to client $k$
\EndFor
\State \Return personalised models $\theta^T_{c(k)}$, $k=1,\dots,K$
\end{algorithmic}
\end{algorithm}

\subsection{Communication, secure aggregation and differential privacy}
\label{sec:comm}

\paragraph{Communication.}
Per round and per client, \ptopofl\ sends $p$ floats up and $p$ floats down, as FedAvg does.
The descriptor adds $46$ floats once. IFCA broadcasts all $M$ cluster models in every round
($Mp$ floats down); SCAFFOLD doubles both directions (Appendix~\ref{app:cost}).

\begin{proposition}[Secure-aggregation compatibility]
\label{prop:secagg}
Suppose that in every round the server learns, for each cluster, only the output of a
secure-aggregation protocol over the clients of that cluster on the inputs
$w_k\theta_k^{t+1}$. Then it can execute Algorithm~\ref{alg:ptopofl} exactly. Apart from the
descriptors, the partition, the weights and the broadcast models, it learns nothing beyond
the per-cluster sums $u_j$.
\end{proposition}
\begin{proof}
Equation~\eqref{eq:round} uses the client models only through $u_j$. The weights $w_k$ are
fixed before training and are known to client $k$, which can therefore submit
$w_k\theta_k^{t+1}$. The second statement is the security guarantee of secure aggregation
\citep{bonawitz2017practical}, applied separately to each cluster.
\end{proof}

Two caveats apply. A cluster with a single client exposes that client's model, so
deployments should enforce $|C_j|\ge2$. And the descriptors are visible to the server unless
they are noised (Proposition~\ref{prop:dp}). Clustering on server-side model comparisons
\citep{sattler2020clustered,long2023multi} is harder to combine with secure aggregation,
because it needs individual updates. IFCA assigns clusters on the client side and can be
combined with secure aggregation by masking per-cluster contributions, at the cost of
broadcasting $M$ models per round.

\begin{proposition}[Record-level differential privacy of everything a client transmits]
\label{prop:dp}
Fix client $k$ and consider neighbouring datasets that differ by the replacement of one
record. Suppose that:
\begin{itemize}[leftmargin=1.2em,itemsep=1pt]
  \item every record is pre-processed so that $\norm{\tilde x_i}\le C$, which implies
    $\norm{\nabla_\theta\ell(\theta;x_i,y_i)}\le C$ for the logistic loss;
  \item every local step adds noise to the gradient,
    $\theta\gets\theta-\eta\bigl(\nabla F_k(\theta)+\tfrac{2C\sigma_{\mathrm{DP}}}{n_k}\xi\bigr)$
    with $\xi\sim\mathcal N(0,I_p)$;
  \item the client releases $\hat\phi_k+\mathcal N(0,\sigma_\phi^2I_{46})$ in place of
    $\hat\phi_k$.
\end{itemize}
Then the sequence of all messages sent by client $k$ over $T$ rounds is
$(\varepsilon,\delta)$-differentially private, with
\[
  \varepsilon=R+2\sqrt{R\log(1/\delta)},\qquad R=\frac{T\tau}{2\sigma_{\mathrm{DP}}^2}+\frac{2}{\sigma_\phi^2}.
\]
Everything the server computes from these messages, and everything other clients receive,
inherits the guarantee by post-processing.
\end{proposition}
The proof (Appendix~\ref{app:proof_dp}) composes Gaussian mechanisms in R\'enyi DP
\citep{mironov2017renyi}. The descriptor term uses only the trivial bound
$\norm{\hat\phi(\calD)-\hat\phi(\calD')}\le2$ on the unit sphere. We do not know of a tighter
provable bound on the per-record sensitivity of this descriptor, consistent with
\citet{kang2024differentially}, and the price of this looseness is quantified in
Section~\ref{sec:e_privacy}. Without noise, the models $\theta_k^t$ are exactly as exposed as
in FedAvg, and the descriptor reveals client-level properties of the data by design.

\section{Theory}
\label{sec:theory}

We analyse the iteration that Algorithm~\ref{alg:ptopofl} actually runs, with clustering,
weighting and blending included. Complete proofs are in Appendix~\ref{app:proofs}.

\subsection{Setting}
\begin{assumption}[Local objectives]\label{ass:sc}
Each $F_k$ is $\mu$-strongly convex and $L$-smooth, with $0<\mu\le L$; write $\kappa=L/\mu$.
\end{assumption}
\begin{assumption}[Fixed structure]\label{ass:fixed}
The partition $\calC=\{C_1,\dots,C_M\}$, the weights $w_k>0$ (with $\sum_{k\in C_j}w_k=1$),
the masses $\pi_j>0$ (with $\sum_j\pi_j=1$) and the coefficient $\alpha\in[0,1]$ do not
change during training.
\end{assumption}
Assumption~\ref{ass:fixed} holds by construction: the structure is computed once, in round
$0$. Every result below holds for \emph{any} structure satisfying it, whether produced by
descriptors, by a random draw, or by an oracle. Topology influences the guarantees only
through the quantities that the structure induces, which we measure.

\paragraph{Notation.}
\begin{itemize}[leftmargin=1.2em,itemsep=2pt]
  \item \emph{Cluster models and norms.} Write $\Theta=(\theta_1,\dots,\theta_M)\in(\R^p)^M$
    and use the weighted norm $\pinorm{\Theta}=(\sum_j\pi_j\norm{\theta_j}^2)^{1/2}$.
  \item \emph{Blending matrix.} $B=(1-\alpha)I_M+\alpha\mathbf{1}\pi^\top$, acting
    blockwise as $(B\Theta)_j=\sum_iB_{ji}\theta_i$.
  \item \emph{Local and round operators.} With $T_k=(\Id-\eta\nabla F_k)^{\tau}$, define
    the cluster operator $S_j(\theta)=\sum_{k\in C_j}w_kT_k(\theta)$, the joint operator
    $S(\Theta)=(S_1(\theta_1),\dots,S_M(\theta_M))$ and the round operator $\Psi=B\circ S$.
    Algorithm~\ref{alg:ptopofl} is the iteration $\Theta^{t+1}=\Psi(\Theta^t)$.
  \item \emph{Contraction factor.} $q:=(1-\eta\mu)^{\tau}$.
  \item \emph{Optima.} $\theta_k^{\loc}=\argmin F_k$ is the local optimum of client $k$.
    $G_j=\sum_{k\in C_j}w_kF_k$ is the cluster objective, with minimiser $\theta_j^\star$,
    and $\Theta^\star=(\theta_1^\star,\dots,\theta_M^\star)$.
  \item \emph{Drift constant.} $\zeta_j=\sum_{k\in C_j}w_k\norm{\nabla F_k(\theta_j^\star)}$.
  \item \emph{Between-cluster dispersion.} For any $\Theta$,
    $D(\Theta)=\pinorm{\Theta-\mathbf{1}\otimes\bar\theta}$ with
    $\bar\theta=\sum_j\pi_j\theta_j$.
\end{itemize}

\subsection{Linear convergence of the complete procedure}

\begin{theorem}[Contraction]\label{thm:contraction}
Under Assumptions~\ref{ass:sc}--\ref{ass:fixed} and $0<\eta\le1/L$, for every $\tau\ge1$ and
every $\alpha\in[0,1]$ the round operator $\Psi$ is a $q$-contraction on
$((\R^p)^M,\pinorm{\cdot})$. Hence it has a unique fixed point
$\hat\Theta=\hat\Theta(\alpha,\eta,\tau)$, and Algorithm~\ref{alg:ptopofl} satisfies
\[
  \pinorm{\Theta^t-\hat\Theta}\le q^t\,\pinorm{\Theta^0-\hat\Theta}=(1-\eta\mu)^{\tau t}\pinorm{\Theta^0-\hat\Theta}.
\]
\end{theorem}
The rate $(1-\eta\mu)^{\tau t}$, that is $(1-1/\kappa)^{\tau t}$ for $\eta=1/L$, is the
usual linear rate of $\tau t$ gradient steps. The content of the theorem is that it holds
for the \emph{whole} procedure, including the clustering, the descriptor-dependent weights
and the blending; $M=1$ recovers weighted FedAvg. The theorem says nothing yet about
\emph{where} the iteration converges, which is the subject of the next results.

\subsection{Where the iteration converges}

\begin{theorem}[Local-step drift]\label{thm:drift}
Under the assumptions of Theorem~\ref{thm:contraction}, take $\alpha=0$. Then for every
cluster $j$,
\[
  \norm{\hat\theta_j-\theta_j^\star}\le c(\eta,\tau)\,\frac{\zeta_j}{\mu},
  \qquad
  c(\eta,\tau)=\min\Bigl\{\tfrac{1}{2}\eta L(\tau-1)(1+\eta\mu\tau),\ \tfrac{1+q}{1-q}\Bigr\}.
\]
In particular, $\hat\theta_j=\theta_j^\star$ when $\tau=1$, and
$\norm{\hat\theta_j-\theta_j^\star}\le\eta L(\tau-1)\zeta_j/\mu$ whenever $\eta\mu\tau\le1$.
\end{theorem}

\begin{proposition}[Effective strength of blending]\label{prop:blend}
Under the assumptions of Theorem~\ref{thm:contraction}, let $\hat\Theta_\alpha$ be the fixed
point for coefficient $\alpha$.
\begin{enumerate}[label=(\alph*),leftmargin=1.6em,itemsep=1pt]
  \item $\pinorm{\hat\Theta_\alpha-\hat\Theta_0}\le\frac{\alpha}{1-q}\,D(\hat\Theta_0)$.
  \item (Exact form for quadratics.) Suppose
    $F_k(\theta)=\tfrac12(\theta-\theta_k^{\loc})^\top H(\theta-\theta_k^{\loc})+c_k$ with a
    Hessian $\mu I\preceq H\preceq LI$ common to all clients, and let $Q=(I-\eta H)^{\tau}$.
    Then $\hat\Theta_0=\Theta^\star$, $\theta_j^\star=\sum_{k\in C_j}w_k\theta_k^{\loc}$, and
    \[
      \hat\theta_{\alpha,j}=\bar m+(1-\alpha)\bigl(I-(1-\alpha)Q\bigr)^{-1}(I-Q)\,(\theta_j^\star-\bar m),
      \qquad\bar m=\textstyle\sum_i\pi_i\theta_i^\star .
    \]
    In an eigendirection of $H$ with eigenvalue $h$, the fraction of the between-cluster
    deviation removed by blending is
    \[
      \alpha_{\mathrm{eff}}(h)=\frac{\alpha}{1-(1-\alpha)q_h},\qquad q_h=(1-\eta h)^{\tau}.
    \]
    This quantity is non-decreasing in $q_h$, and hence at most $\alpha/(1-(1-\alpha)q)$.
  \item (Calibration.) The choice~\eqref{eq:alphacal} gives
    $\alpha_{\mathrm{eff}}(h)\le\bar\alpha$ for every $h\in[\mu,L]$, with equality at
    $h=\mu$.
\end{enumerate}
\end{proposition}

Proposition~\ref{prop:blend} explains why blending must be read together with local
computation. For $\tau=1$ and small $\eta\mu$, any fixed $\alpha>0$ drives every cluster
model to the consensus. In the limit $q\to0$ of exact local solves, $\alpha$ acts at face
value. Blending is the ``adapt-then-combine'' structure of diffusion strategies
\citep{sayed2014adaptation} with a rank-one mixing matrix. The calibration rule removes the
dependence on $(\eta,\tau)$, and Figure~\ref{fig:theory}(c) checks it on the
(non-quadratic) logistic loss.

\subsection{When does clustering lower the error floor?}

The quality of a partition can be measured by the dispersion of the client optima. With the
mixture weights $\omega_k=\pi_{c(k)}w_k$, define
\[
  \Vtot=\sum_k\omega_k\norm{\theta_k^{\loc}-\bar\theta^{\loc}}^2,\qquad
  \Vin(\calC)=\sum_j\pi_j\underbrace{\sum_{k\in C_j}w_k\norm{\theta_k^{\loc}-\bar\theta_j^{\loc}}^2}_{=:V_j},\qquad
  \rho(\calC)=\frac{\Vin(\calC)}{\Vtot}\in[0,1],
\]
where $\bar\theta^{\loc}=\sum_k\omega_k\theta^{\loc}_k$ and
$\bar\theta_j^{\loc}=\sum_{k\in C_j}w_k\theta_k^{\loc}$.

\begin{proposition}[Heterogeneity decomposition]\label{prop:anova}
Under Assumption~\ref{ass:sc}, for every partition and weights:
\begin{enumerate}[label=(\roman*),leftmargin=1.6em,itemsep=1pt]
  \item $\Vtot=\Vin(\calC)+\Vbet(\calC)$ with
    $\Vbet(\calC)=\sum_j\pi_j\norm{\bar\theta_j^{\loc}-\bar\theta^{\loc}}^2\ge0$.
  \item \emph{(Any single model pays the total dispersion.)} For every $\theta\in\R^p$,
    $\sum_k\omega_k\norm{\theta-\theta_k^{\loc}}^2=\norm{\theta-\bar\theta^{\loc}}^2+\Vtot\ge\Vtot$.
  \item \emph{(Cluster optima pay only the within-cluster part.)}
    $\norm{\theta_j^\star-\bar\theta^{\loc}_j}\le\kappa\sqrt{V_j}$, and therefore
    $\sum_k\omega_k\norm{\theta^\star_{c(k)}-\theta_k^{\loc}}^2\le(1+\kappa^2)\,\Vin(\calC)$.
  \item \emph{(Excess risk.)} With $F_k^\star=F_k(\theta_k^{\loc})$, every single model
    satisfies $\sum_k\omega_k[F_k(\theta)-F_k^\star]\ge\frac\mu2\Vtot$, whereas
    $\sum_k\omega_k[F_k(\theta^\star_{c(k)})-F_k^\star]\le\frac L2(1+\kappa^2)\Vin(\calC)$.
    The cluster optima therefore beat every non-personalised model whenever
    $\rho(\calC)<1/(\kappa(1+\kappa^2))$.
  \item \emph{(Common Hessian, exact.)} In the quadratic setting of
    Proposition~\ref{prop:blend}(b),
    \[
      \min_\theta\sum_k\omega_k[F_k(\theta)-F_k^\star]-\sum_k\omega_k[F_k(\theta^\star_{c(k)})-F_k^\star]
      =\tfrac12\tr\bigl(H\Sigma_{\mathrm{bet}}(\calC)\bigr)\ge0,
    \]
    where $\Sigma_{\mathrm{bet}}(\calC)=\sum_j\pi_j(\bar\theta^{\loc}_j-\bar\theta^{\loc})(\bar\theta^{\loc}_j-\bar\theta^{\loc})^\top$.
    The gain is never negative and grows with the between-cluster share.
\end{enumerate}
\end{proposition}

\begin{corollary}[End-to-end personalised error]\label{cor:endtoend}
Under the assumptions of Theorem~\ref{thm:contraction}, for all $t\ge0$,
\[
  \sum_k\omega_k\norm{\theta^t_{c(k)}-\theta_k^{\loc}}^2
  \le
  4q^{2t}\pinorm{\Theta^0-\hat\Theta_\alpha}^2
  +\frac{4\alpha^2}{(1-q)^2}D(\hat\Theta_0)^2
  +4(1+\kappa^2)\bigl(1+c(\eta,\tau)^2\kappa^2\bigr)\,\Vin(\calC).
\]
\end{corollary}
The four contributions are, respectively, an optimisation error that vanishes linearly
(Theorem~\ref{thm:contraction}); a blending shift that vanishes at $\alpha=0$
(Proposition~\ref{prop:blend}); and, in the last term, the drift and the within-cluster
heterogeneity, both of which are proportional to $\Vin(\calC)=\rho(\calC)\Vtot$. By
Proposition~\ref{prop:anova}(ii), no single model can do better than $\Vtot$. Personalisation
therefore lowers the floor in proportion to $\rho(\calC)$. The constants are worst-case over
all strongly convex losses and are pessimistic when $\kappa$ is large; the common-Hessian
identity (v) is exact.

This statement replaces the ``topology--gradient alignment'' assumption of earlier versions.
Whether the descriptor partition helps is equivalent to whether $\rho(\calC_{\mathrm{topo}})$
is small, and $\rho$ can be computed in experiments (Section~\ref{sec:e_alignment}) by
comparing the descriptor partition with random and oracle partitions. Nothing in the
convergence guarantee depends on that comparison.

\subsection{Stochastic local steps: a bias--variance trade-off of clustering}

\begin{theorem}[Stochastic local steps]\label{thm:stochastic}
Under the assumptions of Theorem~\ref{thm:contraction}, suppose that every local step uses
$\nabla F_k(\theta)+\xi$, where the noise is independent across clients and satisfies
$\E[\xi\mid\text{past}]=0$ and $\E[\norm{\xi}^2\mid\text{past}]\le\sigma^2$. Suppose also that
each $\nabla^2F_k$ is $M_H$-Lipschitz. Let $\hat\Theta$ be the deterministic fixed point of
Theorem~\ref{thm:contraction}. Define $W_\pi=\sum_j\pi_j\sum_{k\in C_j}w_k^2$ and
$m_\tau=\sum_{s<\tau}(1-\eta\mu)^{2s}\le\min\{\tau,1/(\eta\mu)\}$. Then
\[
  \E\pinorm{\Theta^t-\hat\Theta}^2\le\Bigl(\tfrac{1+q^2}{2}\Bigr)^{t}\pinorm{\Theta^0-\hat\Theta}^2
  +\frac{2}{1-q^2}\Bigl[W_\pi\,\eta^2\sigma^2m_\tau+\bar\beta^2\,\tfrac{1+q^2}{1-q^2}\Bigr],
  \qquad\bar\beta\le\tfrac14M_H\eta^3\sigma^2\tau^2 .
\]
If every $F_k$ is quadratic, then $\bar\beta=0$, the rate is $q^2$ and the factor $2$
disappears. The leading term satisfies
$W_\pi\eta^2\sigma^2m_\tau/(1-q^2)\le W_\pi\,\eta\sigma^2(1+\eta\mu\tau)/\mu$.
\end{theorem}

Since $\sum_{k\in C_j}w_k^2\ge1/|C_j|$, splitting $K$ clients into $M$ equal clusters with
uniform weights raises $W_\pi$ from $1/K$ to $M/K$, and non-uniform weights raise it further.
Clustering therefore trades the heterogeneity bias of Corollary~\ref{cor:endtoend} for a
larger noise floor: it pays off when $\Vbet(\calC)$ is large relative to the noise. The
logistic loss satisfies the Hessian-Lipschitz condition with
$M_H\le0.097\max_i\norm{\tilde x_i}^3$. Theorem~\ref{thm:stochastic} also covers the DP-GD
updates of Proposition~\ref{prop:dp}, with $\sigma^2=p\,(2C\sigma_{\mathrm{DP}}/\min_kn_k)^2$.

\subsection{Stability of the partition}
\label{sec:theory_partition}

Let $\{v_k\}_{k=1}^K\subset\R^m$ be the vectors that are clustered (here
$v_k=\hat\phi_k$). Define the strict margin of a partition $\calC$ by
\[
  \gamma(\calC)=\min_{i\ne j}\ \min_{k\in C_i,\,l\in C_j}\norm{v_k-v_l}\;-\;\max_j\ \max_{k,l\in C_j}\norm{v_k-v_l}.
\]

\begin{theorem}[Partition stability]\label{thm:partition}
If $\gamma(\calC)>0$, then single-, complete- and average-linkage agglomerative clustering,
cut at $|\calC|$ clusters, all return $\calC$. If moreover $\norm{\tilde v_k-v_k}\le\nu$ for
every $k$ and $4\nu<\gamma(\calC)$, they also return $\calC$ when applied to
$\{\tilde v_k\}$.
\end{theorem}

\begin{corollary}[Gaussian noise on the descriptors]\label{cor:gauss}
Let $\tilde v_k=v_k+\xi_k$ with independent $\xi_k\sim\mathcal N(0,\sigma_\phi^2I_m)$. With
probability at least $1-\beta$, the partition is unchanged whenever
$\sigma_\phi\bigl(\sqrt m+\sqrt{2\log(K/\beta)}\bigr)<\gamma(\calC)/4$.
\end{corollary}

A separation condition on the \emph{average} linkages alone is not sufficient, which is why
both the minimum between-cluster distance and the maximum within-cluster distance appear.
Appendix~\ref{app:counterexample} gives a configuration in which the average-linkage margin
is $15$ and the perturbations have size $0.01$, yet the partition changes. Combining
Theorems~\ref{thm:contraction} and~\ref{thm:partition}: any perturbation of the descriptors
that preserves the partition (subsampling noise, or DP noise satisfying
Corollary~\ref{cor:gauss}) leaves every guarantee of this section unchanged.

\subsection{Trust weights and corrupted clients}
\label{sec:theory_trust}

\begin{proposition}[Weight of corrupted clients inside a cluster]\label{prop:trust}
Let a cluster $C$ contain honest clients $H_C$ and corrupted clients $A_C$, with
$|A_C|=\epsilon|C|$, $0<\epsilon<1$. Assume that $t_k=1$ on $H_C$ and $t_k\le e^{-\Delta}$ on
$A_C$, which holds when $z_k\le z_0$ on $H_C$ and $z_k\ge z_0+\Delta$ on $A_C$. Let
\[
  \Gamma=\frac{\max_{k\in A_C}n_k}{\min_{h\in H_C}n_h}\,
  \exp\Bigl(\max_{h\in H_C}\norm{\hat\phi_h-\bar\phi_C}-\min_{k\in A_C}\norm{\hat\phi_k-\bar\phi_C}\Bigr)\le\frac{\max n}{\min n}\,e^{2}.
\]
Then the corrupted clients receive total weight
\[
  \sum_{k\in A_C}w_k\le\frac{\epsilon\,\Gamma e^{-\Delta}}{(1-\epsilon)+\epsilon\,\Gamma e^{-\Delta}}.
\]
This bound is below $\epsilon$ if and only if $\Gamma e^{-\Delta}<1$, and below $\epsilon^2$
if and only if $\Delta\ge\log\bigl(\Gamma(1+1/\epsilon)\bigr)$.
\end{proposition}

\begin{proposition}[Mean/standard-deviation scores cap the achievable down-weighting]\label{prop:cap}
Let $z_k=(\delta_k-\mathrm{mean}(\delta))/\mathrm{std}(\delta)$ (population standard deviation)
and let $A\subset[K]$ with $|A|=\epsilon K$. Then
$\frac1{|A|}\sum_{k\in A}z_k\le\sqrt{(1-\epsilon)/\epsilon}$. Consequently, under the rule
$t_k=e^{-\max\{z_k-1,0\}}$, a margin $z_k\ge1+\Delta$ on all of $A$ forces
$\Delta\le\sqrt{(1-\epsilon)/\epsilon}-1$, however far the corrupted descriptors lie.
In particular there is no down-weighting at all when $\epsilon\ge1/2$; with $K=8$ and two
corrupted clients, $e^{-\Delta}\ge e^{1-\sqrt3}\approx0.48$.
\end{proposition}

Proposition~\ref{prop:cap} is a version of Samuelson's inequality
\citep{samuelson1968deviant}, and it motivates the robust score of
Section~\ref{sec:structure}. Under that score, if fewer than half of the clients are
corrupted and the honest $\delta_k$ lie in $[a,b]$, then $\MAD(\delta)\le b-a$. Any corrupted
client with $\delta_k\ge D$ then has $z_k\ge(D-b)/\max\{1.4826(b-a),s_0\med(\delta)\}$, which
grows linearly with its separation $D$; so the factor $e^{-\Delta}$ in
Proposition~\ref{prop:trust} can become arbitrarily small. These statements are conditional
on separation \emph{in descriptor space}. Corruption that leaves the feature geometry intact
(label noise, label inversion) produces no separation, and forged descriptors defeat any
descriptor-based rule. Section~\ref{sec:e_robust} measures how often separation occurs in
practice.

\subsection{What the theory does not say}
\label{sec:theory_scope}
\begin{itemize}[leftmargin=1.2em,itemsep=2pt]
  \item The results assume strongly convex and smooth local losses. Non-convex models (deep
    networks) are not covered.
  \item The partition is fixed. Re-clustering during training restarts the analysis, which
    then applies between re-clustering events.
  \item The guarantees hold for any partition. They do \emph{not} imply that descriptor
    partitions are good. That is an empirical property, governed by $\rho(\calC)$ and
    measured in Section~\ref{sec:experiments}.
\end{itemize}

\section{Experiments}
\label{sec:experiments}

Every number, table and figure in this section is produced by a single released script with
fixed seeds (Appendix~\ref{app:repro}).

\subsection{Setup}
\label{sec:exp_setup}

\paragraph{Federations.}
\begin{itemize}[leftmargin=1.2em,itemsep=2pt]
  \item \emph{(A) Synthetic, clustered.} $K=8$ clients, $d=20$ features, and two latent
    groups of four clients. Group-$1$ features lie near a random $3$-dimensional subspace,
    plus isotropic noise of standard deviation $0.3$; group-$0$ features are isotropic
    Gaussian; total variance is matched. Labels follow a logistic model with a
    group-specific direction $\beta_0\pm v$ ($v\perp\beta_0$) and logits rescaled to
    standard deviation~$3$. Each client draws its own prevalence from $\mathcal U(0.10,0.45)$
    and its size from $\{60,\dots,250\}$, and is split $70/30$ (stratified) into training
    and test sets. The controls of Section~\ref{sec:e_alignment} change one ingredient at a
    time:
    \begin{itemize}[leftmargin=1.2em,itemsep=1pt]
      \item \emph{aligned (tails)}: group-$1$ features are Student-$t_3$;
      \item \emph{concept only}: all features are isotropic and the groups differ only in
        $\beta$;
      \item \emph{geometry only}: the groups differ only in geometry and share $\beta$.
    \end{itemize}
  \item \emph{(B) Label skew.} $K=10$ clients with class-conditional Gaussians: positives
    have mean $0.35$ and standard deviation $1.2$ on $10$ informative coordinates, negatives
    mean $0$ and standard deviation $1$. Each client draws its prevalence from
    $\mathcal U(0.1,0.9)$ and holds $n=150$ records. The classes overlap. (In the generator
    of an earlier version of this work they were linearly separable, and every method
    reached AUC $1$.)
  \item \emph{(HD) Fed-Heart-Disease.} Four centres (Cleveland, Budapest, Switzerland, Long
    Beach) from the UCI heart-disease study \citep{detrano1989international}, with the site
    partition of FLamby \citep{terrail2022flamby}. We keep the ten attributes without
    systematic missingness, remove records with missing values, and predict $1[\mathrm{num}>0]$;
    the site sizes are \HDsizes.
\end{itemize}

\paragraph{Model and training.}
All methods train an $\ell_2$-regularised logistic regression ($\lambda=0.01$, all
parameters penalised) with step size $\eta=1/\max_kL_k$, $\tau=10$ full-batch gradient steps
per round, and $T=150$ rounds. Hence $q=\ThQten$ and $\kappa\approx\ThKappa$ on (A). The
local budget is identical across methods, except that pFedMe additionally solves its inner
proximal problems, as its algorithm requires.

\paragraph{Baselines.}
\begin{itemize}[leftmargin=1.2em,itemsep=2pt]
  \item Local training and FedAvg.
  \item FedProx ($\mu=0.1$) and SCAFFOLD (option~II).
  \item pFedMe ($\lambda=15$) and Ditto ($\lambda=0.1$).
  \item IFCA ($M=2$, initialised from two distinct clients' locally trained models).
  \item A \emph{PACFL-style} baseline \citep{vahidian2023efficient}: each client sends the
    top-$3$ left singular vectors of its data matrix, clients are clustered on the sum of
    principal angles, and training is identical to \ptopofl. It therefore isolates the
    effect of the descriptor.
\end{itemize}
Two non-deployable references bracket the role of the partition: the same clustered training
with a \emph{random} partition of identical cluster sizes, and with an \emph{oracle}
partition ($k$-means on the exact local optima, which requires seeing the local models).
Implementation details are in Appendix~\ref{app:baselines}.

\paragraph{Metrics and statistics.}
Each honest client is evaluated on its own held-out split with the model it receives. The
primary metric is the test log-loss, weighted by test size, because it is the quantity
controlled by Corollary~\ref{cor:endtoend}; we also report the pooled AUC. To compare partitions, $\rho(\calC)$ is computed with
sample-proportional weights ($\omega_k=p_k$) for every partition, so that $\Vtot$ is common
and only $\Vin$ varies. Results are means
$\pm$ standard deviations over $\NumSeeds$ seeds, which control data, splits and descriptor
subsamples. We compare \ptopofl\ with each deployable baseline by two-sided paired Wilcoxon
signed-rank tests, Holm-corrected \citep{holm1979simple}; in the tables, $^{+}$ ($^{-}$)
marks baselines significantly worse (better) than \ptopofl\ at level $0.05$. The
hyper-parameters of \ptopofl\ were fixed in advance ($M=2$, $\bar\alpha=0.3$, $z_0=2$) and
were not tuned on test data (Appendix~\ref{app:hparams}).

\subsection{The theory, checked numerically}
\label{sec:e_theory}

\begin{figure}[t]
\centering
\resfig{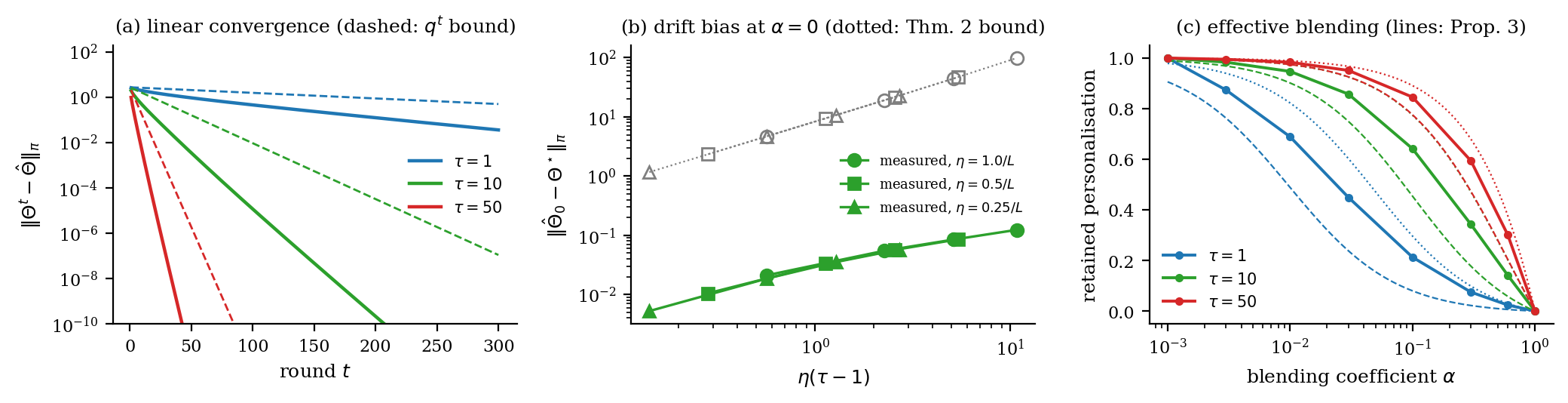}{\textwidth}
\caption{\textbf{Numerical checks of Section~\ref{sec:theory}} on federation (A).
(a) The distance to the fixed point decays at least as fast as $q^t$
(Theorem~\ref{thm:contraction}); dashed lines show the bound.
(b) The drift bias at $\alpha=0$ grows linearly in $\eta(\tau-1)$ and lies below the bound of
Theorem~\ref{thm:drift} (dotted).
(c) Fraction of the between-cluster spread retained by blending; the lines are the
predictions of Proposition~\ref{prop:blend} at the smallest (dashed) and median (dotted)
Hessian eigenvalue.}
\label{fig:theory}
\end{figure}

Figure~\ref{fig:theory} shows three things.
\begin{itemize}[leftmargin=1.2em,itemsep=2pt]
  \item \emph{Contraction (panel a).} For $\tau\in\{1,10,50\}$ the ratio between the
    measured error and the bound $q^t\pinorm{\Theta^0-\hat\Theta}$ never exceeds
    $\ThContractionWorst$.
  \item \emph{Drift (panel b).} The drift bias vanishes at $\tau=1$ up to solver precision
    ($\ThBiasTauOne$), grows linearly in $\eta(\tau-1)$, and stays below the bound of
    Theorem~\ref{thm:drift} in every configuration. The bound is loose by a median factor of
    $\ThBiasTightness$, as expected from its worst-case dependence on $\kappa$.
  \item \emph{Blending (panel c).} The retained personalisation follows
    Proposition~\ref{prop:blend}. With $\tau=10$, the coefficient $\alpha=0.3$ of the
    earlier version removes about $\ThAlphaEffVone$ of the between-cluster spread, whereas
    the calibrated value $\alpha=\ThAlphaCalTen$ removes the intended $\bar\alpha=0.3$. The
    calibration also makes performance insensitive to $\tau$ (Table~\ref{tab:alphatau}).
\end{itemize}

\subsection{Main results}
\label{sec:e_main}

\begin{table}[t]
\centering\small
\caption{\textbf{Personalised performance of honest clients} (mean $\pm$ std over
$\NumSeeds$ seeds; lower log-loss and higher AUC are better). Bold marks the best deployable
method. $^{+}$/$^{-}$: significantly worse/better than \ptopofl\ (paired Wilcoxon, Holm,
$0.05$). $^\diamond$: non-deployable references with the same clustered training as
\ptopofl.}
\label{tab:main}
\resizebox{\textwidth}{!}{\begin{tabular}{lcccccc}
\toprule
 & \multicolumn{2}{c}{Synthetic (clustered)} & \multicolumn{2}{c}{Label skew} & \multicolumn{2}{c}{Fed-Heart-Disease} \\
\cmidrule(lr){2-3} \cmidrule(lr){4-5} \cmidrule(lr){6-7}
Method & Log-loss $\downarrow$ & AUC $\uparrow$ & Log-loss $\downarrow$ & AUC $\uparrow$ & Log-loss $\downarrow$ & AUC $\uparrow$ \\
\midrule
Local & $0.396\pm0.044$ & $0.878\pm0.022$ & $0.625\pm0.050$ & $0.765\pm0.040$ & $0.440\pm0.031$ & $0.879\pm0.016$ \\
FedAvg & $0.451\pm0.050^{+}$ & $0.830\pm0.026$ & $0.590\pm0.018$ & $0.743\pm0.021$ & $0.463\pm0.030^{+}$ & $0.864\pm0.018$ \\
FedProx & $0.450\pm0.050^{+}$ & $0.831\pm0.026$ & $0.589\pm0.017$ & $0.744\pm0.021$ & $0.463\pm0.030^{+}$ & $0.864\pm0.018$ \\
SCAFFOLD & $0.448\pm0.053^{+}$ & $0.831\pm0.027$ & $0.589\pm0.017$ & $0.743\pm0.021$ & $0.459\pm0.027^{+}$ & $0.866\pm0.016$ \\
pFedMe & $0.446\pm0.051^{+}$ & $0.832\pm0.026$ & $0.587\pm0.017$ & $0.746\pm0.021$ & $0.457\pm0.027^{+}$ & $0.867\pm0.017$ \\
Ditto & $0.398\pm0.050$ & $0.872\pm0.027$ & $0.553\pm0.033^{-}$ & $0.785\pm0.032$ & $\mathbf{0.435\pm0.028}$ & $\mathbf{0.880\pm0.016}$ \\
IFCA & $0.393\pm0.067$ & $0.874\pm0.040$ & $\mathbf{0.545\pm0.044}$ & $\mathbf{0.792\pm0.040}$ & $0.441\pm0.030$ & $0.878\pm0.017$ \\
Subspace clustering (PACFL-style) & $0.384\pm0.040$ & $0.882\pm0.021$ & $0.592\pm0.016$ & $0.747\pm0.021$ & $0.451\pm0.031$ & $0.871\pm0.017$ \\
\textbf{\ptopofl} & $\mathbf{0.379\pm0.045}$ & $\mathbf{0.885\pm0.023}$ & $0.592\pm0.018$ & $0.744\pm0.019$ & $0.438\pm0.034$ & $0.879\pm0.019$ \\
\midrule
Clustered, random partition$^\diamond$ & $0.441\pm0.057$ & $0.838\pm0.030$ & $0.593\pm0.014$ & $0.743\pm0.020$ & $0.454\pm0.029$ & $0.869\pm0.016$ \\
Clustered, oracle partition$^\diamond$ & $0.377\pm0.045$ & $0.887\pm0.024$ & $0.537\pm0.039$ & $0.798\pm0.034$ & $0.437\pm0.030$ & $0.880\pm0.017$ \\
\bottomrule
\end{tabular}
}
\end{table}

\begin{figure}[t]
\centering
\resfig{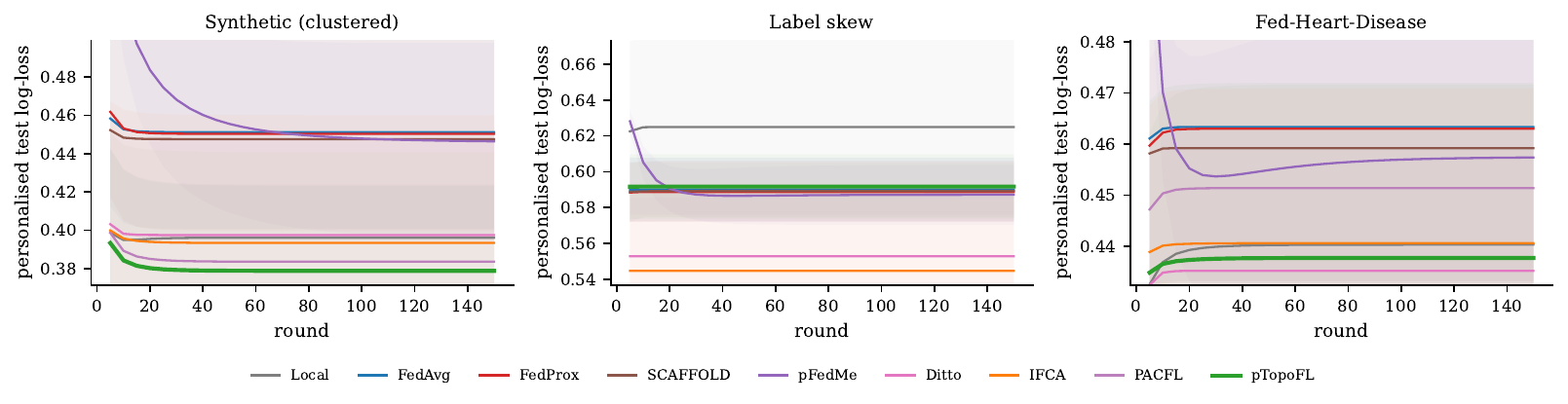}{\textwidth}
\caption{\textbf{Training curves}: personalised test log-loss of honest clients (mean $\pm$
std over seeds). All methods are genuinely iterative and converge within $T=150$ rounds.}
\label{fig:curves}
\end{figure}

\paragraph{Synthetic federation (A).}
\ptopofl\ attains the lowest log-loss ($\ResAPtopoLL$) and the highest AUC
($\ResAPtopoAUC$) of all deployable methods (Table~\ref{tab:main}, Figure~\ref{fig:curves}).
\begin{itemize}[leftmargin=1.2em,itemsep=2pt]
  \item It is significantly better than $\ResANBetter$ of the eight baselines, namely the
    single-model methods (FedAvg $\ResAFedavgLL$, FedProx $\ResAFedproxLL$, SCAFFOLD
    $\ResAScaffoldLL$, pFedMe $\ResAPfedmeLL$).
  \item It is statistically indistinguishable from IFCA ($\ResAIfcaLL$), Ditto
    ($\ResADittoLL$), the PACFL-style baseline ($\ResAPacflLL$) and local training
    ($\ResALocalLL$).
\end{itemize}
The value lies in the partition, not in clustering as such. The descriptor partition
recovers the latent groups (ARI $\ResAARI$) and is as good as the oracle partition in the
sense of Section~\ref{sec:theory}: $\rho(\calC_{\mathrm{topo}})=\ResARhoTopo$, against
$\ResARhoOrac$ for the oracle and $\ResARhoRand$ for a random partition. Correspondingly,
clustered training on a random partition ($\ResARandomLL$) is barely better than FedAvg,
while the oracle partition reaches $\ResAOracleLL$.

\paragraph{Label skew (B).}
Here clients differ through class prevalence. Prevalence changes the feature mixture, but
not in a way the descriptor separates:
$\rho(\calC_{\mathrm{topo}})=\ResBRhoTopo$, no better than a random partition
($\ResBRhoRand$). As Corollary~\ref{cor:endtoend} predicts, \ptopofl\ then performs like
FedAvg ($\ResBPtopoLL$ vs $\ResBFedavgLL$). Methods that personalise through the loss do
better: IFCA ($\ResBIfcaLL$) and Ditto ($\ResBDittoLL$, significantly).

\paragraph{Fed-Heart-Disease (HD).}
On the real federation, personalisation pays off.
\begin{itemize}[leftmargin=1.2em,itemsep=2pt]
  \item \ptopofl\ lowers the log-loss of FedAvg from $\ResHDFedavgLL$ to $\ResHDPtopoLL$ (AUC
    $\ResHDFedavgAUC$ to $\ResHDPtopoAUC$). It is significantly better than $\ResHDNBetter$ of
    the eight baselines, and no baseline is significantly better than it.
  \item Its log-loss is close to that of the best deployable method (\ResHDBest,
    $\ResHDDittoLL$) and of the oracle partition ($\ResHDOracleLL$), and lower than that of
    clustered training on a random partition ($\ResHDRandomLL$) or on the subspace partition
    of the PACFL-style baseline ($\ResHDPacflLL$).
\end{itemize}
The dispersion criterion is only a coarse guide on this federation. With $K=4$ sites and
$M=2$ clusters there are only seven possible partitions. Moreover, the local optima from which
$\rho$ is computed are estimated from as few as about $32$ training records (the Swiss site),
so the measured dispersion also reflects estimation noise. Here
$\rho(\calC_{\mathrm{topo}})=\ResHDRhoTopo$ is \ResHDRhoCmp{} the random-partition value
($\ResHDRhoRand$; oracle $\ResHDRhoOrac$). We therefore read this experiment as evidence that
\ptopofl\ is competitive with the best personalised methods on a real cross-silo federation,
not as a test of the dispersion criterion, which the controlled study of
Section~\ref{sec:e_alignment} addresses.

\paragraph{Per-client view.}
Averages can hide clients that personalisation makes worse off. On (A), \ptopofl\ leaves
$\PCAWorseFedavg$ of the honest clients worse off than FedAvg, the smallest fraction among
personalised methods, and has the best worst-client loss (Table~\ref{tab:perclient},
Appendix~\ref{app:extra}). On (B), where its partition is uninformative, the fraction is
$\PCBWorseFedavg$: the clusters halve the data available to each model without reducing
heterogeneity. This is the bias--variance trade-off of Theorem~\ref{thm:stochastic}. On the real federation (HD), $\PCHDWorseFedavg$ of the sites are worse off than
under FedAvg and $\PCHDWorseLocal$ worse off than with local training.

\subsection{When does topology help?}
\label{sec:e_alignment}

\begin{table}[t]
\centering\small
\caption{\textbf{Alignment between geometry and decision rules} (synthetic, $K=8$).
The first rows show recovery of the latent groups (ARI) and the within-cluster share
$\rho(\calC)$ of the dispersion of client optima (lower is better). The bottom rows show the
personalised log-loss. $^\diamond$: non-deployable.}
\label{tab:alignment}
\resizebox{\textwidth}{!}{\begin{tabular}{lcccc}
\toprule
 & Aligned (dim.) & Aligned (tails) & Concept only & Geometry only \\
\midrule
ARI, PH-descriptor partition & $0.95$ & $1.00$ & $-0.02$ & $0.95$ \\
ARI, subspace partition & $0.87$ & $-0.01$ & $-0.01$ & $0.90$ \\
$\rho(\mathcal{C})$, PH descriptor & $0.46$ & $0.37$ & $0.86$ & $0.78$ \\
$\rho(\mathcal{C})$, random & $0.81$ & $0.84$ & $0.85$ & $0.88$ \\
$\rho(\mathcal{C})$, oracle$^\diamond$ & $0.45$ & $0.37$ & $0.36$ & $0.74$ \\
\midrule
FedAvg & $0.451\pm0.050$ & $0.506\pm0.042$ & $0.477\pm0.028$ & $0.379\pm0.026$ \\
IFCA & $0.393\pm0.067$ & $0.377\pm0.019$ & $0.360\pm0.042$ & $0.374\pm0.033$ \\
Ditto & $0.398\pm0.050$ & $0.400\pm0.023$ & $0.378\pm0.021$ & $0.368\pm0.029$ \\
PACFL-style & $0.384\pm0.040$ & $0.497\pm0.035$ & $0.464\pm0.031$ & $0.383\pm0.029$ \\
\ptopofl & $0.379\pm0.045$ & $0.378\pm0.020$ & $0.469\pm0.038$ & $0.380\pm0.031$ \\
Clustered, random & $0.441\pm0.057$ & $0.494\pm0.042$ & $0.466\pm0.027$ & $0.381\pm0.026$ \\
Clustered, oracle$^\diamond$ & $0.377\pm0.045$ & $0.377\pm0.019$ & $0.347\pm0.027$ & $0.376\pm0.026$ \\
\bottomrule
\end{tabular}
}
\end{table}

\begin{figure}[t]
\centering
\resfig{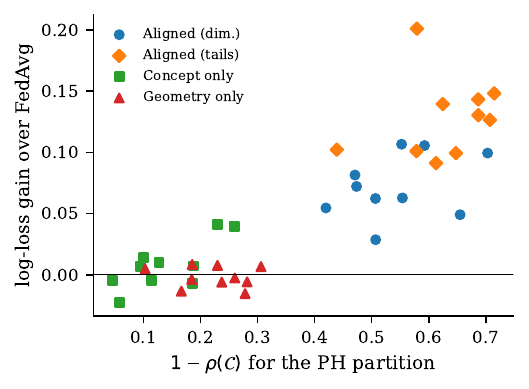}{0.48\textwidth}
\caption{\textbf{Gain over FedAvg against the between-cluster share $1-\rho$} of the
descriptor partition, for every regime and seed. Corollary~\ref{cor:endtoend} predicts a
monotone relation; the correlation is $r=\AlignCorr$.}
\label{fig:alignment}
\end{figure}

Table~\ref{tab:alignment} varies one ingredient at a time.
\begin{itemize}[leftmargin=1.2em,itemsep=2pt]
  \item \emph{Geometry and decision rule co-vary.} When the groups differ in intrinsic
    dimension and in decision rule, both the PH descriptor and the subspace descriptor
    recover the groups (ARI $\AlignAlARI$ and $\AlignAlARIsub$). When they differ in
    \emph{tail behaviour} instead, only the PH descriptor does (ARI $\AlignTaARI$ against
    $\AlignTaARIsub$). In that setting the PACFL-style baseline is no better than FedAvg
    ($\AlignTaPacfl$ vs $\AlignTaFedavg$), while \ptopofl\ matches IFCA and the oracle
    ($\AlignTaPtopo$, $\AlignTaIfca$). This is the case in which a scale-free, non-linear
    shape summary earns its place.
  \item \emph{Concept only.} With identical geometry the descriptor carries no signal (ARI
    $\AlignCoARI$; $\rho=\AlignCoRhoTopo$ against $\AlignCoRhoOrac$ for the oracle), and
    \ptopofl\ is no better than FedAvg ($\AlignCoPtopo$ vs $\AlignCoFedavg$). Loss-based
    personalisation wins clearly (IFCA $\AlignCoIfca$, Ditto $\AlignCoDitto$). This is a
    structural limitation of any feature-only descriptor.
  \item \emph{Geometry only.} The groups are recovered (ARI $\AlignGeARI$), but the optimal
    models barely differ ($\rho_{\mathrm{oracle}}=\AlignGeRhoOrac$). Clustering then neither
    helps nor hurts (gain $\AlignGeGain$).
\end{itemize}
Across all regimes and seeds, the gain over FedAvg tracks the between-cluster share
$1-\rho(\calC_{\mathrm{topo}})$ ($r=\AlignCorr$, Figure~\ref{fig:alignment}), as
Corollary~\ref{cor:endtoend} predicts. The quantity $\rho$ is computable whenever local
optima can be estimated, so a practitioner can check on held-out or simulated data whether
the descriptor partition is informative before deploying it.

\subsection{Clients with corrupted data}
\label{sec:e_robust}

\begin{figure}[t]
\centering
\resfig{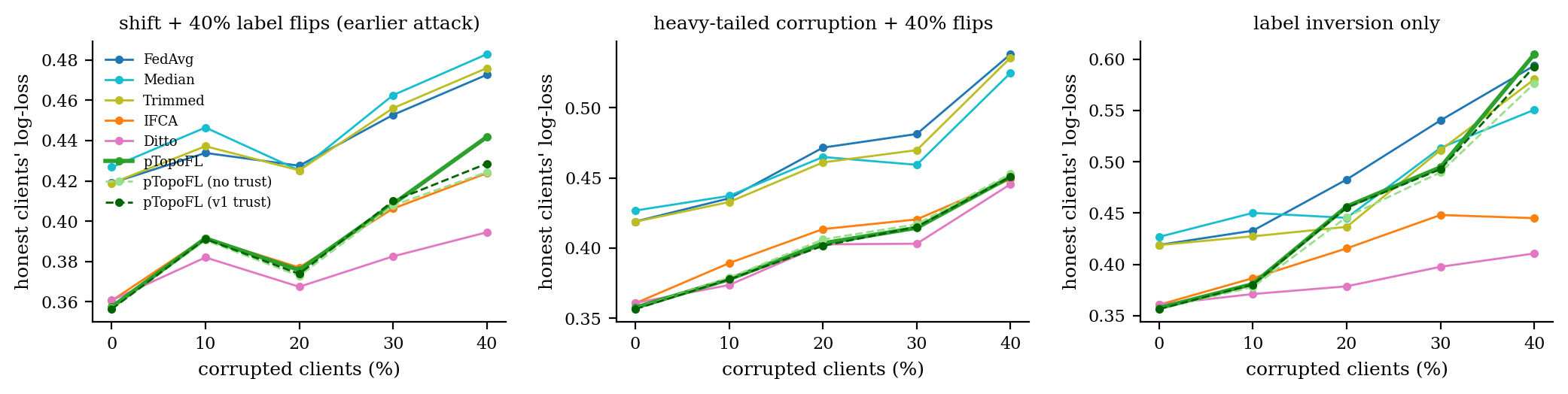}{\textwidth}
\caption{\textbf{Corrupted clients}: log-loss of the \emph{honest} clients as the fraction
of corrupted clients grows ($K=10$, federation A). Dashed lines are \ptopofl\ variants
without trust weights and with the mean/std trust rule.}
\label{fig:robust}
\end{figure}

Following the threat model of Section~\ref{sec:framework}, corrupted clients follow the
protocol but hold corrupted data. We use three corruptions:
\begin{itemize}[leftmargin=1.2em,itemsep=2pt]
  \item the attack of the earlier version: a feature shift plus $40\%$ label flips;
  \item heavy-tailed feature corruption plus $40\%$ label flips;
  \item label inversion alone.
\end{itemize}
The robust aggregators coordinate-wise median and trimmed mean \citep{yin2018byzantine} are
added as baselines. Figure~\ref{fig:robust} and Table~\ref{tab:detection} (Appendix) show
three findings.
\begin{itemize}[leftmargin=1.2em,itemsep=2pt]
  \item \emph{Feature corruption.} When corruption alters the features, \ptopofl\ protects
    honest clients much better than FedAvg and the robust aggregators. With
    $\RobFracMid\%$ heavy-tailed corruption it reaches $\RobGePtopo$, against
    $\RobGeFedavg$ for FedAvg and $\RobGeMedian$ for the median, on par with Ditto
    ($\RobGeDitto$) and IFCA ($\RobGeIfca$).
  \item \emph{The protection does not come from trust weights.} Descriptor-based detection
    is weak: averaged over attacks and fractions, the robust score flags
    $\RobTPR$ of the corrupted clients and $\RobFPR$ of the honest ones. Removing trust
    changes little (shift-plus-flips attack, $\RobFracMid\%$: $\RobPoPtopoNT$ without trust
    against $\RobPoPtopo$ with it). The protection comes from personalisation: corrupted
    clients mostly harm the cluster that hosts them.
  \item \emph{Label inversion} leaves the geometry intact. \ptopofl\ then degrades like
    FedAvg, and slightly more at $\RobFracMax\%$ ($\RobLfPtopoMax$ against
    $\RobLfFedavgMax$); Ditto is the most robust method ($\RobLfDittoMax$).
\end{itemize}
These results support Section~\ref{sec:theory_trust}. The guarantees on trust weights are
conditional on separation in descriptor space, which corrupted clients often do not exhibit
when the honest population is itself heterogeneous. We therefore do not claim adversarial
robustness, and we recommend Byzantine-robust aggregation where it matters.

\subsection{Privacy: what the descriptor costs and what attacks see}
\label{sec:e_privacy}

\begin{figure}[t]
\centering
\resfig{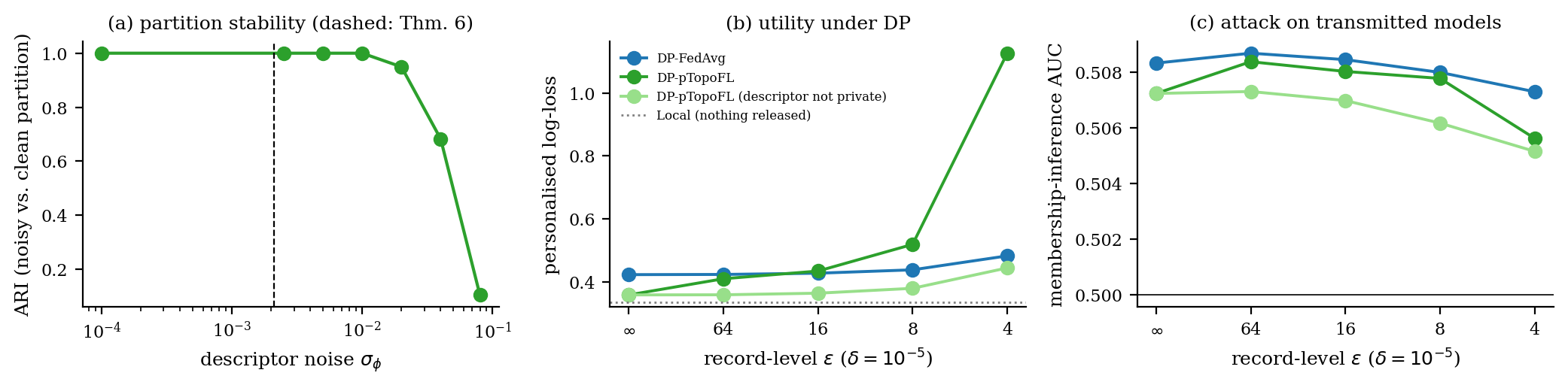}{\textwidth}
\caption{\textbf{Privacy.} (a) Agreement between the partitions obtained from noisy and
clean descriptors, as a function of the noise level; the dashed line is the radius of
Corollary~\ref{cor:gauss}. (b) Personalised log-loss under record-level
$(\varepsilon,10^{-5})$-DP (Proposition~\ref{prop:dp}). (c) Membership-inference AUC on the
models clients transmit.}
\label{fig:privacy}
\end{figure}

\paragraph{How much noise does the partition tolerate?}
The partition is preserved (ARI $\ge0.9$) up to $\sigma_\phi=\DPsigmaOK$
(Figure~\ref{fig:privacy}a). The radius of Corollary~\ref{cor:gauss} is conservative
($\DPsafeSigma$), and a strictly positive margin exists in $\DPgapPos$ of the seeds. Under
the worst-case sensitivity $2$ of Proposition~\ref{prop:dp}, a noise level of
$\sigma_\phi=\DPsigmaOK$ corresponds to $\varepsilon\approx\DPepsOK$ for the descriptor
release alone \citep[analytic Gaussian mechanism,][]{balle2018improving}, which is not a
meaningful guarantee. For context only: removing one record changes $\hat\phi_k$ by a median
of $\DPlooMed$ (maximum $\DPlooMax$ over the records tried), whereas the median distance
between clients is $\DPinterMed$. This is an empirical observation, not a sensitivity bound.

\paragraph{Record-level DP for everything a client transmits.}
We run DP-GD with the accounting of Proposition~\ref{prop:dp} on federation (A) with larger
clients ($n_k\in[400,1200]$, $C=\DPclip$, $T\tau=\DPsteps$ noisy steps) and compare three
variants:
\begin{itemize}[leftmargin=1.2em,itemsep=2pt]
  \item DP-FedAvg, which spends the whole budget on gradients;
  \item DP-\ptopofl, which spends half of the budget on the descriptor;
  \item DP-\ptopofl\ with a non-private descriptor. This variant protects only the
    models: the guarantee then covers $\theta_k$ and not $\phi_k$.
\end{itemize}
The results are in Figure~\ref{fig:privacy}b and Table~\ref{tab:privacy}.
\begin{itemize}[leftmargin=1.2em,itemsep=2pt]
  \item \emph{Without privacy}, \ptopofl\ improves on FedAvg ($\DPPtopoLLinf$ vs
    $\DPFedavgLLinf$).
  \item \emph{With the full guarantee}, at $\varepsilon=\DPepsMid$ the descriptor noise
    ($\sigma_\phi=\DPsigPhiMid$) erases the partition (ARI $\DPPtopoARImid$), and DP-\ptopofl\
    ($\DPPtopoLLmid$) is worse than DP-FedAvg ($\DPFedavgLLmid$).
  \item \emph{If only the models are protected}, the partition keeps its value
    ($\DPPtopoNPLLmid$).
\end{itemize}
Local training, which releases nothing, reaches $\DPLocalLL$ here because these clients are
large.

\paragraph{Attacks.}
\begin{itemize}[leftmargin=1.2em,itemsep=2pt]
  \item \emph{Membership inference.} A loss-threshold attack \citep{yeom2018privacy} on the
    last model each client transmits achieves an AUC of $\DPFedavgMIAinf$ (FedAvg) and
    $\DPPtopoMIAinf$ (\ptopofl) even without DP, close to chance for this regularised,
    low-capacity model. Its weakness should not be read as evidence of privacy: stronger
    attacks and higher-capacity models leak more.
  \item \emph{Property inference.} The descriptor itself leaks population-level properties
    \citep{melis2019exploiting}. From clean descriptors, a logistic classifier recovers a
    client's latent geometry group with leave-one-out accuracy $\PropInfClean$. This falls to
    $\PropInfMid$, no better than chance, only at the noise level that also destroys the
    partition.
\end{itemize}

In short, the descriptor is useful exactly because it reveals client-level geometry. With
the only sensitivity bound we can prove, protecting it under record-level DP removes its
usefulness. Deployments must decide whether population-level geometry is sensitive.
Low-sensitivity topological summaries \citep{kang2024differentially} are the natural next
step.

\subsection{Ablations and sensitivity}
\label{sec:e_ablation}

\begin{table}[t]
\centering\small
\caption{\textbf{Ablation} on federation (A) (log-loss, mean $\pm$ std over seeds).}
\label{tab:ablation}
\begin{tabular}{lc}
\toprule
Variant (Synthetic, clustered) & Log-loss $\downarrow$ \\
\midrule
Full \ptopofl{} ($\bar\alpha=0.3$, $\tau=10$) & $0.379\pm0.045$ \\
No clustering ($M=1$) & $0.455\pm0.053$ \\
Uniform within-cluster weights ($w_k\propto n_k$) & $0.378\pm0.044$ \\
No trust weights & $0.378\pm0.044$ \\
v1 trust rule (mean/std $z$) & $0.380\pm0.044$ \\
Uncalibrated $\alpha=0.3$ (v1 default) & $0.405\pm0.049$ \\
$\alpha=0$ (no blending) & $0.377\pm0.044$ \\
$\alpha=1$ (consensus only) & $0.455\pm0.053$ \\
\bottomrule
\end{tabular}

\end{table}

Table~\ref{tab:ablation} shows four things.
\begin{itemize}[leftmargin=1.2em,itemsep=2pt]
  \item \emph{Clustering is the active ingredient.} With $M=1$ the method reduces to
    weighted FedAvg ($\AblNoClust$ against $\AblFull$).
  \item \emph{Topology weights and trust weights have no measurable effect on clean data}
    (uniform weights $\AblUnif$, no trust $\AblNoTrust$). We keep them, but their benefit
    is not demonstrated here.
  \item \emph{Calibration matters.} The uncalibrated $\alpha=0.3$ of the earlier version
    costs $\AblVoneAlpha$ against $\AblFull$. The calibrated value is close to $\alpha=0$
    ($\AblAlphaZero$), and $\alpha=1$ falls back to consensus ($\AblAlphaOne$).
  \item \emph{Insensitivity to $\tau$.} With calibration, performance is essentially
    insensitive to $\tau$ (Table~\ref{tab:alphatau}).
\end{itemize}
The sensitivity analysis (Table~\ref{tab:sensitivity}) shows that the method is insensitive
to $M\in\{2,3,4\}$ and to $n_{\mathrm{sub}}\in\{40,80,160\}$. Averaging several subsamples
matters ($n_{\mathrm{rep}}=1$ degrades group recovery), and $\ell_2$ normalisation of the
descriptor works better than per-coordinate standardisation. A drift-monitoring experiment,
in which the descriptors of two clients change progressively, is reported in
Appendix~\ref{app:extra} (detection AUC $\ContAUC$).

\section{Related work}
\label{sec:related}

\paragraph{Clustered and personalised FL.}
IFCA \citep{ghosh2020efficient} alternates client-side cluster assignment by loss with
per-cluster training. CFL \citep{sattler2020clustered} and hierarchical clustering of updates
\citep{briggs2020federated} split clients by the geometry of their updates, and FeSEM
\citep{long2023multi} by distances between models. These signals are model-dependent and
recomputed during training. One-shot alternatives summarise each client's data once: local
$k$-means centroids in $k$-FED \citep{dennis2021heterogeneity}, and principal subspaces in
PACFL \citep{vahidian2023efficient}, the closest prior work to ours. Our experiments show
that persistent homology and linear subspaces see different things (Table~\ref{tab:alignment}).
Personalisation without clustering includes Moreau-envelope methods
\citep{t2020personalized}, proximal personal models \citep{li2021ditto}, model mixtures
\citep{deng2020adaptive,hanzely2020federated} and prototype sharing \citep{tan2022fedproto}.

\paragraph{Theory of local steps.}
Analyses of local SGD under heterogeneity
\citep{stich2019local,khaled2020tighter,koloskova2020unified,woodworth2020minibatch,karimireddy2020scaffold}
and fixed-point analyses \citep{malinovsky2020local,charles2021convergence,pathak2020fedsplit}
show that local steps bias the limit. We extend the fixed-point view to clustered, weighted
and blended aggregation, and give an exact formula for the effective strength of blending;
blending is an adapt-then-combine diffusion step \citep{sayed2014adaptation}.

\paragraph{Topology in machine learning and in FL.}
PH has been used in centralised learning for graphs, time series, regularisation and medical
images
\citep{hofer2017deep,umeda2017time,chen2019topological,clough2020topological,abaach2023geotop,gouiaa2024taelcore},
with software such as GUDHI, Ripser and giotto-tda \citep{gudhi,ripser,tauzin2021giotto}.
In FL, persistence diagrams of model weights have been used to detect poisoning
\citep{ma2024federated,alipourjeddi2025topology}, and topological embeddings of learned
representations have been used to align clients \citep{hu2025fedtopo}. We instead use
descriptors of each client's \emph{input} distribution, computed once, to structure the
federation, and we analyse the resulting algorithm.

\paragraph{Privacy and robustness.}
Gradient inversion \citep{zhu2019deep,geiping2020inverting}, membership inference
\citep{yeom2018privacy} and property inference \citep{melis2019exploiting} motivate DP
\citep{dwork2014algorithmic,abadi2016deep,wei2020federated} and secure aggregation
\citep{bonawitz2017practical}. DP for topological summaries is studied by
\citet{kang2024differentially}. Byzantine-robust aggregation
\citep{blanchard2017machine,yin2018byzantine} addresses adversaries that our descriptor-based
trust cannot.

\paragraph{FL in healthcare.}
Multi-site clinical data motivate FL \citep{rieke2020future}. FLamby
\citep{terrail2022flamby} provides realistic cross-silo benchmarks, including the
heart-disease partition used here, and topological methods have been applied to
lung-transplant outcomes \citep{tran2025topoattention}.

\section{Discussion and limitations}
\label{sec:discussion}

\ptopofl\ is a clustered FL method whose partition is fixed in advance from feature
geometry. Our analysis shows that such a method converges linearly for any partition. It
also shows that the benefit of the partition is governed by a single measurable quantity
$\rho(\calC)$, and that blending must be calibrated against local computation. Experiments
show when the PH descriptor provides a good partition: when clients that differ in
non-linear, scale-free aspects of feature geometry also differ in the model they need. They
also show when it does not.

\paragraph{Limitations.}
\begin{enumerate}[label=(\roman*),leftmargin=1.6em,itemsep=2pt]
  \item \emph{Convexity.} The theory assumes strongly convex and smooth losses. Deep models,
    where descriptors could be computed on learned representations, are not covered, and
    we report no deep-learning experiments.
  \item \emph{Fixed partition.} The partition is fixed, and re-clustering under drift is
    not analysed.
  \item \emph{What the descriptor sees.} The descriptor ignores labels, location and scale.
    Concept shift, label shift and mean shifts are invisible to it, and loss-based methods
    (IFCA, Ditto) should be preferred there. Its components are not Lipschitz functions of
    the data, so there is no data-level stability guarantee.
  \item \emph{Corrupted clients.} Descriptor-based trust detected few corrupted clients, it
    can down-weight honest but unusual sites, and it is powerless against forged
    descriptors.
  \item \emph{Privacy.} \ptopofl\ is no more private than FedAvg, and its descriptor leaks
    population-level properties. Record-level DP that covers the descriptor destroys the
    partition under the sensitivity bound we can prove.
  \item \emph{Evidence base.} The evidence consists of synthetic federations and one small
    real federation ($K=4$). Larger real benchmarks are needed.
\end{enumerate}

\section{Conclusion}
\label{sec:conclusion}

We studied topology-guided clustered federated learning, in which a persistent-homology
descriptor of each client's features, computed once, fixes which clients share a model. For
the complete algorithm we proved linear convergence and characterised the fixed point
explicitly, through local-step drift, the effective strength of blending with a calibration
rule, and a stochastic floor. We also showed that personalisation lowers the heterogeneity
floor in proportion to a measurable between-cluster share. Experiments confirm these
predictions. The descriptor partition is as good as an oracle partition when feature geometry
and decision rules co-vary, it captures shape differences that linear subspaces miss, and it
offers no advantage over FedAvg when heterogeneity is invisible in the features. We hope that
this explicit delimitation, which states both where the descriptor helps and where it does
not, makes topological summaries a well-understood tool for structuring federations.

\section*{Broader Impact Statement}
\ptopofl\ targets settings such as multi-site healthcare. Three risks deserve emphasis.
\begin{itemize}[leftmargin=1.2em,itemsep=2pt]
  \item \emph{It is not a privacy technology.} Without DP and secure aggregation, model
    updates are as exposed as in FedAvg, and the descriptor reveals population-level
    properties of a site, such as the presence of distinct subpopulations. It must not be
    presented as protecting patient privacy.
  \item \emph{Trust weights carry a fairness risk.} They can down-weight legitimate sites
    with unusual populations, for example small hospitals specialising in rare conditions.
    Deployments should audit false-flag rates per site, which we report, and floor the
    trust weights.
  \item \emph{Personalisation can disadvantage some clients.} Personalised models can be
    worse than the global model for some clients, and per-client results, not only
    averages, should be inspected before deployment.
\end{itemize}

\appendix
\section{Proofs}
\label{app:proofs}

Throughout, $\norm{\cdot}$ is the Euclidean norm and $\pinorm{\cdot}$ the weighted norm of
Section~\ref{sec:theory}.

\subsection{Auxiliary lemmas}

\begin{lemma}[Gradient-step map]\label{lem:gd}
Let $F$ be convex and $L$-smooth, and let $G=\Id-\eta\nabla F$.
\begin{enumerate}[label=(\roman*),leftmargin=1.6em,itemsep=1pt]
  \item If $F$ is also $\mu$-strongly convex and $0<\eta\le2/(\mu+L)$, then
    $\norm{G(x)-G(y)}\le(1-\eta\mu)\norm{x-y}$ for all $x,y$.
  \item If $0<\eta\le2/L$, then $\norm{\nabla F(G(x))}\le\norm{\nabla F(x)}$ for all $x$.
  \item If $F$ is $\mu$-strongly convex with minimiser $x^\star$, then
    $\norm{x-x^\star}\le\norm{\nabla F(x)}/\mu$.
\end{enumerate}
\end{lemma}
\begin{proof}
(i) Write $g_x=\nabla F(x)$. For $\mu$-strongly convex, $L$-smooth $F$ we have
$\langle g_x-g_y,x-y\rangle\ge\frac{\mu L}{\mu+L}\norm{x-y}^2+\frac{1}{\mu+L}\norm{g_x-g_y}^2$
\citep[Thm.~2.1.12]{nesterov2018lectures}. Therefore
\[
  \norm{G(x)-G(y)}^2\le\Bigl(1-\tfrac{2\eta\mu L}{\mu+L}\Bigr)\norm{x-y}^2+\eta\Bigl(\eta-\tfrac{2}{\mu+L}\Bigr)\norm{g_x-g_y}^2 .
\]
The last coefficient is non-positive, and $\norm{g_x-g_y}\ge\mu\norm{x-y}$ by strong
convexity. Hence
$\norm{G(x)-G(y)}^2\le\bigl(1-\frac{2\eta\mu L}{\mu+L}+\eta^2\mu^2-\frac{2\eta\mu^2}{\mu+L}\bigr)\norm{x-y}^2=(1-\eta\mu)^2\norm{x-y}^2$,
and $1-\eta\mu\ge0$ because $\eta\le2/(\mu+L)\le1/\mu$.

(ii) Let $y=G(x)=x-\eta g_x$. Co-coercivity
\citep[Thm.~2.1.5]{nesterov2018lectures} gives
$\langle g_x-g_y,\eta g_x\rangle\ge\frac1L\norm{g_x-g_y}^2$. Hence
$\norm{g_y}^2=\norm{g_x}^2-2\langle g_x,g_x-g_y\rangle+\norm{g_x-g_y}^2\le\norm{g_x}^2-\bigl(\frac{2}{\eta L}-1\bigr)\norm{g_x-g_y}^2\le\norm{g_x}^2$.

(iii) We have $\mu\norm{x-x^\star}^2\le\langle\nabla F(x)-\nabla F(x^\star),x-x^\star\rangle\le\norm{\nabla F(x)}\norm{x-x^\star}$.
\end{proof}

Under Assumption~\ref{ass:sc} and $\eta\le1/L\le2/(\mu+L)$, Lemma~\ref{lem:gd}(i) applied
$\tau$ times shows that $T_k$ is $q$-Lipschitz, $q=(1-\eta\mu)^\tau$. Since
$T_k(\theta_k^{\loc})=\theta_k^{\loc}$, this gives
\begin{equation}
  \norm{T_k(\theta)-\theta_k^{\loc}}\le q\,\norm{\theta-\theta_k^{\loc}}\qquad\forall\theta.
  \label{eq:Tk_loc}
\end{equation}

\begin{lemma}[Blending]\label{lem:blend}
Let $P=\mathbf{1}\pi^\top$, so that $(P\Theta)_j=\bar\theta=\sum_i\pi_i\theta_i$ for every
$j$. Then:
\begin{itemize}[leftmargin=1.2em,itemsep=1pt]
  \item $P$ is the orthogonal projector, for the inner product
    $\langle\Theta,\Theta'\rangle_\pi=\sum_j\pi_j\langle\theta_j,\theta'_j\rangle$, onto
    consensus configurations;
  \item $B=P+(1-\alpha)(I-P)$;
  \item $\pinorm{B\Theta}\le\pinorm{\Theta}$ for every $\Theta$.
\end{itemize}
\end{lemma}
\begin{proof}
$P^2=P$ because $\pi^\top\mathbf 1=1$. Moreover
$\langle P\Theta,\Theta'\rangle_\pi=\langle\bar\theta,\bar\theta'\rangle=\langle\Theta,P\Theta'\rangle_\pi$,
so $P$ is self-adjoint. The identity $B=P+(1-\alpha)(I-P)$ follows from the definitions.
Finally,
$\pinorm{B\Theta}^2=\pinorm{P\Theta}^2+(1-\alpha)^2\pinorm{(I-P)\Theta}^2\le\pinorm{\Theta}^2$.
\end{proof}

\begin{lemma}[Weighted variance decomposition]\label{lem:anova}
For weights $a_k\ge0$ with $\sum_ka_k=1$, vectors $v_k$ with mean $\bar v=\sum_ka_kv_k$, and
any vector $c$,
\[
  \sum_ka_k\norm{c-v_k}^2=\norm{c-\bar v}^2+\sum_ka_k\norm{v_k-\bar v}^2 .
\]
\end{lemma}
\begin{proof}
Expand $\norm{(c-\bar v)+(\bar v-v_k)}^2$ and use $\sum_ka_k(\bar v-v_k)=0$.
\end{proof}

\subsection{Proof of Theorem~\ref{thm:contraction}}
For any $\Theta,\Theta'$, using that each $T_k$ is $q$-Lipschitz and that the weights in a
cluster sum to one,
\[
  \pinorm{S(\Theta)-S(\Theta')}^2=\sum_j\pi_j\Bigl\|\sum_{k\in C_j}w_k\bigl(T_k(\theta_j)-T_k(\theta'_j)\bigr)\Bigr\|^2
  \le\sum_j\pi_j\Bigl(\sum_{k\in C_j}w_k\,q\norm{\theta_j-\theta'_j}\Bigr)^2=q^2\pinorm{\Theta-\Theta'}^2 .
\]
By Lemma~\ref{lem:blend},
$\pinorm{\Psi(\Theta)-\Psi(\Theta')}=\pinorm{B(S(\Theta)-S(\Theta'))}\le q\pinorm{\Theta-\Theta'}$.
Since $q<1$, Banach's fixed-point theorem on the complete space
$((\R^p)^M,\pinorm\cdot)$ gives existence and uniqueness of $\hat\Theta$. The rate follows
from $\pinorm{\Theta^{t}-\hat\Theta}=\pinorm{\Psi(\Theta^{t-1})-\Psi(\hat\Theta)}\le q\pinorm{\Theta^{t-1}-\hat\Theta}$.
\qed

\subsection{Proof of Theorem~\ref{thm:drift}}
Fix a cluster $j$ and write $S=S_j$, $\theta^\star=\theta^\star_j$, $\hat\theta=\hat\theta_j$;
with $\alpha=0$, $\hat\theta=S(\hat\theta)$.

\emph{First bound.} Since $S$ is $q$-Lipschitz,
$\norm{\hat\theta-\theta^\star}\le\norm{S(\hat\theta)-S(\theta^\star)}+\norm{S(\theta^\star)-\theta^\star}\le q\norm{\hat\theta-\theta^\star}+\norm{S(\theta^\star)-\theta^\star}$,
hence
\begin{equation}
  \norm{\hat\theta-\theta^\star}\le\frac{\norm{S(\theta^\star)-\theta^\star}}{1-q}.
  \label{eq:drift1}
\end{equation}
For $k\in C_j$ let $x_k^0=\theta^\star$ and $x_k^{s+1}=x_k^s-\eta\nabla F_k(x_k^s)$, so that
$T_k(\theta^\star)-\theta^\star=-\eta\sum_{s=0}^{\tau-1}\nabla F_k(x_k^s)$. Because
$\sum_{k\in C_j}w_k\nabla F_k(\theta^\star)=\nabla G_j(\theta^\star)=0$,
\[
  S(\theta^\star)-\theta^\star=-\eta\sum_{s=0}^{\tau-1}\sum_{k\in C_j}w_k\bigl[\nabla F_k(x_k^s)-\nabla F_k(\theta^\star)\bigr].
\]
By $L$-smoothness and Lemma~\ref{lem:gd}(ii),
\[
  \norm{\nabla F_k(x_k^s)-\nabla F_k(\theta^\star)}\le L\norm{x_k^s-\theta^\star}
  \le L\eta\sum_{r<s}\norm{\nabla F_k(x_k^r)}\le L\eta s\norm{\nabla F_k(\theta^\star)} .
\]
Summing over $s<\tau$ gives $\norm{S(\theta^\star)-\theta^\star}\le\frac12\eta^2L\tau(\tau-1)\zeta_j$.
Next, $(1-x)^\tau\le e^{-\tau x}\le(1+\tau x)^{-1}$ for $x\in[0,1]$, so
$1-q\ge\eta\mu\tau/(1+\eta\mu\tau)$. Substituting both estimates into~\eqref{eq:drift1}
yields $\norm{\hat\theta-\theta^\star}\le\frac12\eta L(\tau-1)(1+\eta\mu\tau)\,\zeta_j/\mu$.

\emph{Second bound.} Since $\sum_kw_k=1$,
$\hat\theta-\theta^\star=\sum_{k}w_k\bigl(T_k(\hat\theta)-\theta_k^{\loc}\bigr)+\sum_kw_k(\theta^{\loc}_k-\theta^\star)$.
By~\eqref{eq:Tk_loc} and Lemma~\ref{lem:gd}(iii) applied to $F_k$,
$\norm{\theta_k^{\loc}-\theta^\star}\le\norm{\nabla F_k(\theta^\star)}/\mu$. Hence
\[
  \norm{\hat\theta-\theta^\star}\le q\sum_kw_k\bigl(\norm{\hat\theta-\theta^\star}+\norm{\theta^\star-\theta_k^{\loc}}\bigr)+\frac{\zeta_j}{\mu}
  \le q\norm{\hat\theta-\theta^\star}+(1+q)\frac{\zeta_j}{\mu},
\]
which gives the second bound. For $\tau=1$ the first bound vanishes. If
$\eta\mu\tau\le1$, then $(1+\eta\mu\tau)/2\le1$.
\qed

\subsection{Proof of Proposition~\ref{prop:blend}}
\emph{(a)} Since $\hat\Theta_0=S(\hat\Theta_0)$ and $\hat\Theta_\alpha=BS(\hat\Theta_\alpha)$,
\[
  \hat\Theta_\alpha-\hat\Theta_0=B\bigl(S(\hat\Theta_\alpha)-S(\hat\Theta_0)\bigr)+(B-I)\hat\Theta_0 .
\]
By Lemma~\ref{lem:blend}, $(B-I)\hat\Theta_0=-\alpha(I-P)\hat\Theta_0$, whose norm is
$\alpha D(\hat\Theta_0)$. Hence
$\pinorm{\hat\Theta_\alpha-\hat\Theta_0}\le q\pinorm{\hat\Theta_\alpha-\hat\Theta_0}+\alpha D(\hat\Theta_0)$.

\emph{(b)} Here $\nabla F_k(\theta)=H(\theta-\theta^{\loc}_k)$, so
$T_k(\theta)=\theta_k^{\loc}+Q(\theta-\theta^{\loc}_k)$ and, with
$\bar w_j=\sum_{k\in C_j}w_k\theta_k^{\loc}$, $S_j(\theta)=\bar w_j+Q(\theta-\bar w_j)$. Since
$\nabla G_j(\theta)=H(\theta-\bar w_j)$, we get $\theta_j^\star=\bar w_j$; and since
$I-Q\succ0$, $S_j$ has the unique fixed point $\bar w_j$, so $\hat\Theta_0=\Theta^\star$.

For general $\alpha$, let $m=\sum_j\pi_j\hat\theta_j$. Applying $\pi^\top$ to
$\hat\Theta=BS(\hat\Theta)$ and using $\pi^\top B=\pi^\top$ gives
$m=\sum_j\pi_jS_j(\hat\theta_j)=\bar m+Q(m-\bar m)$, hence $m=\bar m$. Projecting on the
disagreement subspace with Lemma~\ref{lem:blend},
\[
  \hat\theta_j-m=(1-\alpha)\bigl(S_j(\hat\theta_j)-m\bigr)=(1-\alpha)\bigl[(I-Q)(\bar w_j-\bar m)+Q(\hat\theta_j-m)\bigr],
\]
which is the stated formula; the inverse exists because $\norm{(1-\alpha)Q}<1$. In an
eigendirection with eigenvalue $h$ the matrix acts as the scalar
$(1-\alpha)(1-q_h)/(1-(1-\alpha)q_h)=1-\alpha_{\mathrm{eff}}(h)$. Finally,
$\partial\alpha_{\mathrm{eff}}/\partial q_h=\alpha(1-\alpha)/(1-(1-\alpha)q_h)^2\ge0$, and
$q_h\le q$ for $h\ge\mu$.

\emph{(c)} Substituting~\eqref{eq:alphacal} into $\alpha/(1-(1-\alpha)q)$ gives exactly
$\bar\alpha$; monotonicity in $q_h$ then gives the inequality.
\qed

\subsection{Proof of Proposition~\ref{prop:anova} and Corollary~\ref{cor:endtoend}}
\emph{(i)} Apply Lemma~\ref{lem:anova} with weights $w_k$ inside each cluster and $c=\bar\theta^{\loc}$:
$\sum_{k\in C_j}w_k\norm{\bar\theta^{\loc}-\theta_k^{\loc}}^2=\norm{\bar\theta^{\loc}-\bar\theta^{\loc}_j}^2+V_j$.
Multiply by $\pi_j$ and sum over $j$.

\emph{(ii)} This is Lemma~\ref{lem:anova} with weights $\omega_k$.

\emph{(iii)} $G_j$ is $\mu$-strongly convex with $\nabla G_j(\theta_j^\star)=0$, and
$\nabla F_k(\theta_k^{\loc})=0$. By Lemma~\ref{lem:gd}(iii), smoothness and Jensen,
\[
  \mu\norm{\bar\theta^{\loc}_j-\theta_j^\star}\le\Bigl\|\sum_{k\in C_j}w_k\bigl(\nabla F_k(\bar\theta^{\loc}_j)-\nabla F_k(\theta^{\loc}_k)\bigr)\Bigr\|
  \le L\sum_{k\in C_j}w_k\norm{\bar\theta^{\loc}_j-\theta^{\loc}_k}\le L\sqrt{V_j}.
\]
Lemma~\ref{lem:anova} with $c=\theta_j^\star$ gives
$\sum_{k\in C_j}w_k\norm{\theta_j^\star-\theta_k^{\loc}}^2=\norm{\theta_j^\star-\bar\theta_j^{\loc}}^2+V_j\le(1+\kappa^2)V_j$.
Multiply by $\pi_j$ and sum.

\emph{(iv)} Strong convexity and smoothness give
$\frac\mu2\norm{\theta-\theta_k^{\loc}}^2\le F_k(\theta)-F_k^\star\le\frac L2\norm{\theta-\theta_k^{\loc}}^2$.
Combine with (ii) and (iii). The cluster optima beat every single model if
$\frac L2(1+\kappa^2)\Vin<\frac\mu2\Vtot$, that is, if $\rho(\calC)<1/(\kappa(1+\kappa^2))$.

\emph{(v)} Here
$F_k(\theta)-F_k^\star=\frac12(\theta-\theta_k^{\loc})^\top H(\theta-\theta_k^{\loc})$.
Minimising $\sum_k\omega_k(\cdot)$ over $\theta$ gives $\theta=\bar\theta^{\loc}$ and the value
$\frac12\tr(H\Sigma_{\mathrm{tot}})$, with
$\Sigma_{\mathrm{tot}}=\sum_k\omega_k(\theta_k^{\loc}-\bar\theta^{\loc})(\theta_k^{\loc}-\bar\theta^{\loc})^\top$.
Since $\theta_j^\star=\bar\theta^{\loc}_j$ by Proposition~\ref{prop:blend}(b), the cluster
optima achieve $\frac12\tr(H\Sigma_{\mathrm{in}})$, with
$\Sigma_{\mathrm{in}}=\sum_j\pi_j\sum_{k\in C_j}w_k(\theta_k^{\loc}-\bar\theta_j^{\loc})(\cdot)^\top$.
The matrix version of (i), $\Sigma_{\mathrm{tot}}=\Sigma_{\mathrm{in}}+\Sigma_{\mathrm{bet}}$,
concludes the proof.

\emph{Corollary~\ref{cor:endtoend}.} For client $k$ in cluster $j$, write
\[
  \theta^t_j-\theta^{\loc}_k=(\theta^t_j-\hat\theta_{\alpha,j})+(\hat\theta_{\alpha,j}-\hat\theta_{0,j})+(\hat\theta_{0,j}-\theta^\star_j)+(\theta^\star_j-\theta^{\loc}_k),
\]
and use $\norm{a+b+c+d}^2\le4(\norm a^2+\norm b^2+\norm c^2+\norm d^2)$. For any
cluster-level quantity $x_j$, $\sum_k\omega_k\norm{x_{c(k)}}^2=\sum_j\pi_j\norm{x_j}^2$. The
four terms are then bounded as follows:
\begin{itemize}[leftmargin=1.2em,itemsep=1pt]
  \item the first by Theorem~\ref{thm:contraction};
  \item the second by Proposition~\ref{prop:blend}(a);
  \item the third by Theorem~\ref{thm:drift}, together with
    $\zeta_j\le L\sum_{k\in C_j}w_k\norm{\theta_j^\star-\theta_k^{\loc}}\le L\sqrt{(1+\kappa^2)V_j}$
    (from (iii)), so that $\sum_j\pi_j\norm{\hat\theta_{0,j}-\theta_j^\star}^2\le c(\eta,\tau)^2\kappa^2(1+\kappa^2)\Vin$;
  \item the fourth by (iii).
\end{itemize}
\qed

\subsection{Proof of Theorem~\ref{thm:stochastic}}
Let $\mathcal F_t$ be the history up to the start of round $t$. Write $T_k^\xi$ for the noisy
local map, $D_t=\Theta^t-\hat\Theta$, and $E_t=S^\xi(\Theta^t)-S(\Theta^t)$. Then
$D_{t+1}=A_t+BE_t$ with $A_t=B(S(\Theta^t)-S(\hat\Theta))$, which is
$\mathcal F_t$-measurable and satisfies $\pinorm{A_t}\le q\pinorm{D_t}$.

\emph{Coupling.} Run the noisy and noiseless local trajectories $x_s,y_s$ of client $k$ from
the same starting point, and let $d_s=x_s-y_s$. Then $d_{s+1}=G(x_s)-G(y_s)-\eta\xi_s$, and
the cross term vanishes because $\E[\xi_s\mid x_s]=0$. Lemma~\ref{lem:gd}(i) therefore gives
$\E\norm{d_{s+1}}^2\le(1-\eta\mu)^2\E\norm{d_s}^2+\eta^2\sigma^2$, hence
$\E\norm{e_k}^2\le\eta^2\sigma^2m_\tau$ for $e_k=T_k^\xi(\theta)-T_k(\theta)$.

\emph{Bias.} Taylor's theorem with an $M_H$-Lipschitz Hessian gives
$d_{s+1}=(I-\eta\nabla^2F(y_s))d_s-\eta r_s-\eta\xi_s$ with
$\norm{r_s}\le\frac{M_H}{2}\norm{d_s}^2$. Since $0\preceq I-\eta\nabla^2F\preceq(1-\eta\mu)I$,
\[
  \norm{\E d_{s+1}}\le(1-\eta\mu)\norm{\E d_s}+\tfrac{\eta M_H}{2}\E\norm{d_s}^2\le(1-\eta\mu)\norm{\E d_s}+\tfrac{\eta^3M_H\sigma^2s}{2},
\]
so that $\norm{\E[e_k\mid\mathcal F_t]}\le\frac14M_H\eta^3\sigma^2\tau^2=:\bar\beta$. For
quadratics $r_s=0$, hence $\bar\beta=0$.

\emph{One round.} The local noises are independent across clients. Hence, by
Lemma~\ref{lem:blend},
\[
  \E\pinorm{BE_t}^2\le\E\pinorm{E_t}^2\le\sum_j\pi_j\Bigl[\sum_{k\in C_j}w_k^2\,\E\norm{e_k}^2+\bar\beta^2\Bigr]\le W_\pi\eta^2\sigma^2m_\tau+\bar\beta^2,
\]
and $|\E\langle A_t,BE_t\rangle_\pi|\le q\,\E\pinorm{D_t}\,\bar\beta$. With
$2ab\le\varepsilon a^2+b^2/\varepsilon$ and $\varepsilon=(1-q^2)/(2q^2)$,
\[
  \E\pinorm{D_{t+1}}^2\le\tfrac{1+q^2}{2}\,\E\pinorm{D_t}^2+W_\pi\eta^2\sigma^2m_\tau+\bar\beta^2\,\tfrac{1+q^2}{1-q^2}.
\]
Unrolling this recursion gives the claim. For quadratics the cross term is zero, which gives
the rate $q^2$. The leading term is bounded using
$m_\tau\le\tau$ and $1-q^2\ge1-q\ge\eta\mu\tau/(1+\eta\mu\tau)$.
\qed

\subsection{Proof of Theorem~\ref{thm:partition} and Corollary~\ref{cor:gauss}}
\label{app:proof_partition}
\emph{Theorem~\ref{thm:partition}.} Let
$d_{\mathrm{in}}=\max_j\max_{k,l\in C_j}\norm{v_k-v_l}$ and
$d_{\mathrm{out}}=d_{\mathrm{in}}+\gamma(\calC)$. For each of the three linkages, the linkage
between two disjoint sets $U,V$ lies in
$[\min_{u\in U,v\in V}\norm{u-v},\ \max_{u\in U,v\in V}\norm{u-v}]$.

We show by induction that, until $|\calC|$ clusters remain, every current cluster is
contained in some $C_j$. This holds initially, when every cluster is a singleton. While more
than $|\calC|$ clusters remain, by the pigeonhole principle some $C_j$ contains two current
clusters, whose linkage is at most $d_{\mathrm{in}}$. Every pair of current clusters lying
in different $C_j$'s has linkage at least $d_{\mathrm{out}}>d_{\mathrm{in}}$. The merge of
minimal linkage therefore joins two subsets of the same $C_j$, which preserves the
invariant. When $|\calC|$ clusters remain, each is contained in some $C_j$ and they cover
$[K]$, so they are exactly the $C_j$.

Under perturbation, every pairwise distance changes by at most $2\nu$, so the perturbed
margin is at least $\gamma(\calC)-4\nu>0$, and the first part applies.

\emph{Corollary~\ref{cor:gauss}.} The norm is $1$-Lipschitz and
$\E\norm{\xi_k}\le\sigma_\phi\sqrt m$, so Gaussian concentration
\citep{boucheron2013concentration} gives
$\Pr\bigl(\norm{\xi_k}\ge\sigma_\phi(\sqrt m+s)\bigr)\le e^{-s^2/2}$. Take
$s=\sqrt{2\log(K/\beta)}$ and a union bound over the $K$ clients.
\qed

\subsection{A counterexample for average-linkage margins}
\label{app:counterexample}
Take four points on a line, $v=(0,\,10,\,20.01,\,30.01)$, with $\calC=\{\{1,2\},\{3,4\}\}$.
\begin{itemize}[leftmargin=1.2em,itemsep=1pt]
  \item The within-cluster average linkage, including the diagonal as in the earlier
    version, is $5$, and the between-cluster average linkage is $20.01$. The ``average
    margin'' is therefore $15.01$, far larger than $4\nu=0.04$ for $\nu=0.01$.
  \item Unperturbed, average linkage returns $\calC$.
  \item Moving point $2$ by $+0.01$ and point $3$ by $-0.01$ makes the pair $\{2,3\}$ the
    closest one ($9.99<10.01$). The cut at two clusters becomes $\{\{1,2,3\},\{4\}\}$.
\end{itemize}
The strict margin $\gamma(\calC)$ of Theorem~\ref{thm:partition} is $0.01-10<0$ here,
correctly signalling instability.

\subsection{Proof of Propositions~\ref{prop:trust} and~\ref{prop:cap}}
\emph{Proposition~\ref{prop:trust}.} Write $r_k=\norm{\hat\phi_k-\bar\phi_C}$. The weight of
$A_C$ is $a/(a+h)$ with $a=\sum_{k\in A_C}n_ke^{-r_k}t_k$ and $h=\sum_{l\in H_C}n_le^{-r_l}$.
This expression increases in $a$ and decreases in $h$. Moreover
$a\le|A_C|\max_{A_C}n\,e^{-\min_{A_C}r}e^{-\Delta}$ and
$h\ge|H_C|\min_{H_C}n\,e^{-\max_{H_C}r}$. Substituting gives the bound, since
$|A_C|/|H_C|=\epsilon/(1-\epsilon)$. The bound $\Gamma\le\frac{\max n}{\min n}e^2$ holds
because $\norm{\hat\phi_k-\bar\phi_C}\le\norm{\hat\phi_k}+\norm{\bar\phi_C}\le2$. The two
threshold statements follow by solving $x/((1-\epsilon)+x)<\epsilon$ and
$x/((1-\epsilon)+x)\le\epsilon^2$ for $x=\epsilon\Gamma e^{-\Delta}$, which give
$\Gamma e^{-\Delta}<1$ and $\Gamma e^{-\Delta}\le\epsilon/(1+\epsilon)$ respectively.

\emph{Proposition~\ref{prop:cap}.} The scores satisfy $\sum_kz_k=0$ and $\sum_kz_k^2=K$. Let
$s=\sum_{k\in A}z_k$, so that $\sum_{k\notin A}z_k=-s$. By Cauchy--Schwarz,
\[
  K\ge\frac{s^2}{|A|}+\frac{s^2}{K-|A|}=\frac{s^2K}{|A|(K-|A|)},
\]
hence $s/|A|\le\sqrt{(K-|A|)/|A|}=\sqrt{(1-\epsilon)/\epsilon}$. If $z_k\ge1+\Delta$ on all
of $A$, then $1+\Delta\le s/|A|$.
\qed

\subsection{Proof of Proposition~\ref{prop:dp}}
\label{app:proof_dp}
Under replace-one adjacency, the sum of per-record gradients changes by at most $2C$ in
$\ell_2$ norm. Each noisy step therefore releases a Gaussian mechanism whose ratio of
sensitivity to noise standard deviation is $2C/(2C\sigma_{\mathrm{DP}})=1/\sigma_{\mathrm{DP}}$.
Its R\'enyi DP of order $a>1$ is $a/(2\sigma_{\mathrm{DP}}^2)$ \citep{mironov2017renyi}.
The descriptor lies on the unit sphere, so its sensitivity is at most $2$, and its release
has R\'enyi DP $a\cdot4/(2\sigma_\phi^2)$.

R\'enyi DP composes additively, also adaptively: the starting points of local steps depend
on earlier releases, including those of other clients, which are independent of client
$k$'s data. Over the $T\tau$ steps and the single descriptor release, the total is $aR$ with
$R=T\tau/(2\sigma_{\mathrm{DP}}^2)+2/\sigma_\phi^2$. Converting to $(\varepsilon,\delta)$-DP
\citep[Prop.~3]{mironov2017renyi} gives $\varepsilon=\min_{a>1}\{aR+\log(1/\delta)/(a-1)\}$.
The minimum is attained at $a=1+\sqrt{\log(1/\delta)/R}$, where it equals
$R+2\sqrt{R\log(1/\delta)}$.

The transmitted models are deterministic functions of the noisy gradients, the starting
points and $n_k$, which is treated as public. The server's outputs and other clients' views
are post-processing. Features must be pre-processed by a record-wise map, such as the norm
rescaling used here, or standardised with public constants; pooled standardisation
statistics would be an additional release.
\qed

\section{Algorithms and baselines}
\label{app:baselines}
All methods start from $\theta=0$ and use full-batch gradients, the step size
$\eta=1/\max_kL_k$, $\tau=10$ local steps per round and $T=150$ rounds. They are
evaluated as in Section~\ref{sec:exp_setup}. The baselines implement their published
update rules.
\begin{itemize}[leftmargin=1.2em,itemsep=2pt]
  \item \textbf{FedAvg}: averaging with $p_k\propto n_k$; this is Algorithm~\ref{alg:ptopofl} with $M=1$ and $w_k=p_k$.
  \item \textbf{FedProx}: local objective $F_k(\theta)+\frac{\mu}{2}\norm{\theta-\theta^t}^2$, $\mu=0.1$, step size $1/(L+\mu)$.
  \item \textbf{SCAFFOLD}: option~II control variates, global step size $1$, full participation, weights $p_k$ \citep{karimireddy2020scaffold}.
  \item \textbf{pFedMe} \citep{t2020personalized}: $\lambda=15$.
    \begin{itemize}[leftmargin=1.2em,itemsep=1pt]
      \item Each of the $\tau$ local iterations approximately solves the proximal problem
        $\min_\theta F_k(\theta)+\frac\lambda2\norm{\theta-w_k}^2$ with $10$ warm-started
        gradient steps of size $1/(L+\lambda)$, then sets
        $w_k\gets w_k-\tfrac12(w_k-\tilde\theta_k)$.
      \item The server averages with $\beta=1$.
      \item A client's personalised model is the proximal point of the final global model.
    \end{itemize}
  \item \textbf{Ditto} \citep{li2021ditto}: a FedAvg global model, plus personal models
    updated with $\tau$ steps per round on $F_k(v)+\frac{\lambda}{2}\norm{v-w^t}^2$,
    $\lambda=0.1$.
  \item \textbf{IFCA} \citep{ghosh2020efficient}: $M=2$.
    \begin{itemize}[leftmargin=1.2em,itemsep=1pt]
      \item The cluster models are initialised from two distinct, randomly chosen clients,
        each trained locally for $20$ steps; identical initialisations would make every
        client choose the same cluster.
      \item In each round, every client selects the cluster model with the lowest training
        loss; each cluster then averages its members' local models.
    \end{itemize}
  \item \textbf{PACFL-style}: each client sends the top-$3$ left singular vectors of its
    (features $\times$ samples) training matrix \citep{vahidian2023efficient}. The
    dissimilarity is the sum of the principal angles, and clustering uses average linkage
    with $M=2$. Training is identical to \ptopofl, with $w_k\propto n_k$ and the same
    calibrated $\alpha$, so the comparison isolates the descriptor.
  \item \textbf{Robust aggregators}: unweighted coordinate-wise median and trimmed mean
    (trimming fraction $0.2$) of the clients' local models \citep{yin2018byzantine}.
  \item \textbf{References.} The random partition is a uniformly random permutation of the
    \ptopofl\ labels, so cluster sizes are identical. The oracle partition is $k$-means
    ($20$ restarts) on the exact local optima, computed by Newton's method. Both use the
    \ptopofl\ training with $w_k\propto n_k$.
\end{itemize}

\section{Descriptor details, implementation check and cost}
\label{app:cost}
\paragraph{Layout.} Coordinates $0$--$19$ hold $\mathrm{BC}_0$, $20$--$39$ hold
$\mathrm{BC}_1$, and coordinates $40$--$45$ hold $(N_0,N_1,E_0,E_1,A_0,A_1)$. Records with
zero pairwise distance, which occur with discrete attributes, produce pairs of zero
persistence; these are excluded from $N_j$, $E_j$, $A_j$ and the Betti curves. Constant
coordinates are left unchanged by the min--max normalisation.

\paragraph{Correctness of $H_1$.} The released code computes $H_1$ with Ripser when it is
installed, and otherwise with a direct $\mathbb{Z}_2$ boundary-matrix reduction of the same
truncated triangle filtration. On random point clouds the two give identical pairs, up to
$6\cdot10^{-8}$ (single-precision rounding in Ripser).

\paragraph{Cost.} With Ripser, one subsample of $80$ points in $\R^{20}$ takes about
$\CostDescOne$\,s on one CPU core; the pure-NumPy $H_1$ reduction takes about
$\CostHoneNumpy$\,s. The descriptor is computed once per client. Its cost is therefore
negligible compared with training, even with $n_{\mathrm{rep}}=5$ subsamples.
Table~\ref{tab:comm} gives the communication per round.

\begin{table}[h]
\centering\small
\caption{Communication per client and per round (in floats); $p$ is the model dimension.}
\label{tab:comm}
\begin{tabular}{lcc}
\toprule
Method & uplink / client / round & downlink / client / round \\
\midrule
FedAvg, FedProx, Ditto & $p$ & $p$ \\
SCAFFOLD & $2p$ & $2p$ \\
pFedMe & $p$ & $p$ \\
IFCA ($M$ clusters) & $p$ (+ cluster id) & $Mp$ \\
\ptopofl & $p$ (+ $46$ floats once) & $p$ \\
\bottomrule
\end{tabular}

\end{table}

\section{Hyper-parameters}
\label{app:hparams}
\begin{table}[h]
\centering\small
\caption{Hyper-parameters, fixed before the experiments and shared by all federations.}
\label{tab:hparams}
\begin{tabular}{ll}
\toprule
Hyper-parameter & Value \\
\midrule
$\lambda$ (L2, all parameters) & 0.01 \\
$\tau$ (local GD steps) & 10 \\
$T$ (rounds) & 150 \\
$\eta$ & $1/\max_k L_k$ \\
$M$ (clusters) & 2 \\
$\bar\alpha$ (target effective blending) & 0.3 \\
$n_{\mathrm{sub}}$, subsamples averaged & 80, 5 \\
$L$ (Betti-curve resolution) & 20 \\
$z_0$ (robust $z$ threshold) & 2.0 \\
FedProx $\mu$ & 0.1 \\
pFedMe $\lambda$, inner steps & 15.0, 10 \\
Ditto $\lambda$ & 0.1 \\
test fraction & 0.3 \\
\bottomrule
\end{tabular}

\end{table}
\paragraph{Synthetic generators.}
\begin{itemize}[leftmargin=1.2em,itemsep=2pt]
  \item Federation (A): $d=20$; the subspace dimension is $3$; the concept shift has unit
    norm along a direction orthogonal to $\beta_0$; logits are rescaled to standard
    deviation $3$.
  \item Corrupted clients: the shift-plus-flips attack adds $\mathcal N(2,0.5^2)$ to $10$
    coordinates and flips $40\%$ of the labels; the heavy-tailed attack adds Student-$t_2$
    noise to all coordinates and flips $40\%$ of the labels.
  \item Privacy experiment: $n_k\in[400,1200]$, rows rescaled to $\norm{\tilde x}\le C=\DPclip$,
    no pooled standardisation (it would be an additional release), $T=40$, $\tau=5$ and
    $\delta=10^{-5}$.
\end{itemize}

\section{Additional results}
\label{app:extra}

\begin{table}[h]
\centering\small
\caption{\textbf{Per-client view} (all federations). Shown are the fraction of honest
clients whose personalised log-loss is worse than under FedAvg and under local training, and
the loss of the worst client, all averaged over seeds.}
\label{tab:perclient}
\resizebox{\textwidth}{!}{\begin{tabular}{lccccccccc}
\toprule
 & \multicolumn{3}{c}{Synthetic (clustered)} & \multicolumn{3}{c}{Label skew} & \multicolumn{3}{c}{Fed-Heart-Disease} \\
\cmidrule(lr){2-4} \cmidrule(lr){5-7} \cmidrule(lr){8-10}
Method & worse than FedAvg & worse than Local & worst client & worse than FedAvg & worse than Local & worst client & worse than FedAvg & worse than Local & worst client \\
\midrule
Local & 0.26 & 0.00 & 0.601 & 0.65 & 0.00 & 0.820 & 0.42 & 0.00 & 0.567 \\
FedAvg & 0.00 & 0.74 & 0.598 & 0.00 & 0.35 & 0.723 & 0.00 & 0.57 & 0.586 \\
IFCA & 0.16 & 0.42 & 0.519 & 0.30 & 0.17 & 0.695 & 0.38 & 0.62 & 0.576 \\
Ditto & 0.17 & 0.51 & 0.518 & 0.33 & 0.06 & 0.690 & 0.33 & 0.53 & 0.549 \\
PACFL-style & 0.11 & 0.45 & 0.520 & 0.51 & 0.34 & 0.741 & 0.42 & 0.65 & 0.573 \\
\ptopofl & 0.06 & 0.41 & 0.501 & 0.63 & 0.33 & 0.727 & 0.33 & 0.45 & 0.547 \\
\bottomrule
\end{tabular}
}
\end{table}

\begin{table}[h]
\centering\small
\caption{\textbf{Blending and local steps} (federation A, log-loss). Rows: number of local
steps $\tau$, with the total number of local steps held equal. Columns: target effective
blending $\bar\alpha$, with $\alpha$ set by Eq.~\eqref{eq:alphacal}. With calibration, the
results depend on $\bar\alpha$ but hardly on $\tau$, as Proposition~\ref{prop:blend}(c)
predicts.}
\label{tab:alphatau}
\begin{tabular}{lccccc}
\toprule
$\tau$ \textbackslash{} $\bar\alpha$ & 0 & 0.1 & 0.3 & 0.6 & 1 \\
\midrule
1 & $0.377$ & $0.378$ & $0.379$ & $0.383$ & $0.448$ \\
10 & $0.377$ & $0.378$ & $0.379$ & $0.383$ & $0.455$ \\
50 & $0.379$ & $0.379$ & $0.380$ & $0.388$ & $0.479$ \\
\bottomrule
\end{tabular}

\end{table}

\begin{table}[h]
\centering\small
\caption{\textbf{Sensitivity} of \ptopofl\ on federation (A): log-loss and recovery of the
latent groups (ARI).}
\label{tab:sensitivity}
\begin{tabular}{lcc}
\toprule
Variant & Log-loss $\downarrow$ & ARI \\
\midrule
$M=1$ & $0.455\pm0.053$ & $0.00$ \\
$M=2$ (default) & $0.379\pm0.045$ & $0.95$ \\
$M=3$ & $0.383\pm0.048$ & $0.73$ \\
$M=4$ & $0.378\pm0.048$ & $0.52$ \\
$n_{\mathrm{sub}}=40$ & $0.379\pm0.047$ & $0.95$ \\
$n_{\mathrm{sub}}=160$ & $0.380\pm0.045$ & $0.95$ \\
one subsample ($n_{\mathrm{rep}}=1$) & $0.392\pm0.054$ & $0.80$ \\
per-coordinate $z$-score normalisation & $0.397\pm0.057$ & $0.68$ \\
\bottomrule
\end{tabular}

\end{table}

\begin{table}[h]
\centering\small
\caption{\textbf{Detection of corrupted clients} by trust weights (federation A, $K=10$).
Each entry gives the fraction of corrupted clients flagged ($t_k<1$) and the fraction of
honest clients flagged, for each fraction of corrupted clients.}
\label{tab:detection}
\begin{tabular}{lcccc}
\toprule
Attack / score & 10\% & 20\% & 30\% & 40\% \\
\midrule
shift + flips, robust & 0.00 / 0.10 & 0.05 / 0.21 & 0.07 / 0.10 & 0.03 / 0.23 \\
shift + flips, mean/std & 0.10 / 0.14 & 0.05 / 0.21 & 0.10 / 0.19 & 0.05 / 0.25 \\
heavy-tailed + flips, robust & 0.30 / 0.02 & 0.35 / 0.04 & 0.23 / 0.06 & 0.15 / 0.12 \\
heavy-tailed + flips, mean/std & 0.50 / 0.13 & 0.45 / 0.07 & 0.37 / 0.17 & 0.15 / 0.25 \\
label inversion, robust & 0.00 / 0.12 & 0.10 / 0.09 & 0.17 / 0.19 & 0.10 / 0.12 \\
label inversion, mean/std & 0.00 / 0.17 & 0.20 / 0.16 & 0.20 / 0.16 & 0.23 / 0.23 \\
\bottomrule
\end{tabular}

\end{table}

\begin{table}[h]
\centering\small
\caption{\textbf{Privacy experiment} (Section~\ref{sec:e_privacy}): log-loss,
membership-inference AUC, partition recovery and property-inference accuracy as functions
of the record-level $\varepsilon$ ($\delta=10^{-5}$).}
\label{tab:privacy}
\resizebox{\textwidth}{!}{\begin{tabular}{lccccc}
\toprule
$\varepsilon$ & $\infty$ & 64 & 16 & 8 & 4 \\
\midrule
DP-FedAvg log-loss & $0.423$ & $0.424$ & $0.428$ & $0.438$ & $0.483$ \\
DP-\ptopofl{} log-loss & $0.359$ & $0.410$ & $0.434$ & $0.519$ & $1.127$ \\
DP-\ptopofl{} (descriptor not private) log-loss & $0.359$ & $0.359$ & $0.364$ & $0.379$ & $0.444$ \\
DP-FedAvg MIA AUC & $0.508$ & $0.509$ & $0.508$ & $0.508$ & $0.507$ \\
DP-\ptopofl{} MIA AUC & $0.507$ & $0.508$ & $0.508$ & $0.508$ & $0.506$ \\
DP-\ptopofl{} (descriptor not private) MIA AUC & $0.507$ & $0.507$ & $0.507$ & $0.506$ & $0.505$ \\
DP-\ptopofl{} ARI (partition vs.\ latent groups) & $1.00$ & $0.05$ & $0.00$ & $0.00$ & $0.00$ \\
Property inference from descriptor (accuracy) & $1.00$ & $0.47$ & $0.40$ & $0.39$ & $0.39$ \\
\bottomrule
\end{tabular}
}
\end{table}

\paragraph{Drift monitoring.}
When client data change over time, descriptors can be recomputed periodically to monitor
drift. In this experiment, $K=6$ clients receive fresh samples in each of $20$ rounds, and
two of them drift progressively from isotropic to near-$3$-dimensional data. The statistic
$\Delta_k=\frac1R\sum_r\norm{\hat\phi_k^{(r)}-\hat\phi_k^{(1)}}$ separates drifting from
stable clients (AUC $\ContAUC$; flagging the top quartile gives TPR $\ContTPR$ and FPR
$\ContFPR$; Figure~\ref{fig:continual}). With static data, descriptors are computed once,
so no spurious drift arises.

\begin{figure}[h]
\centering
\resfig{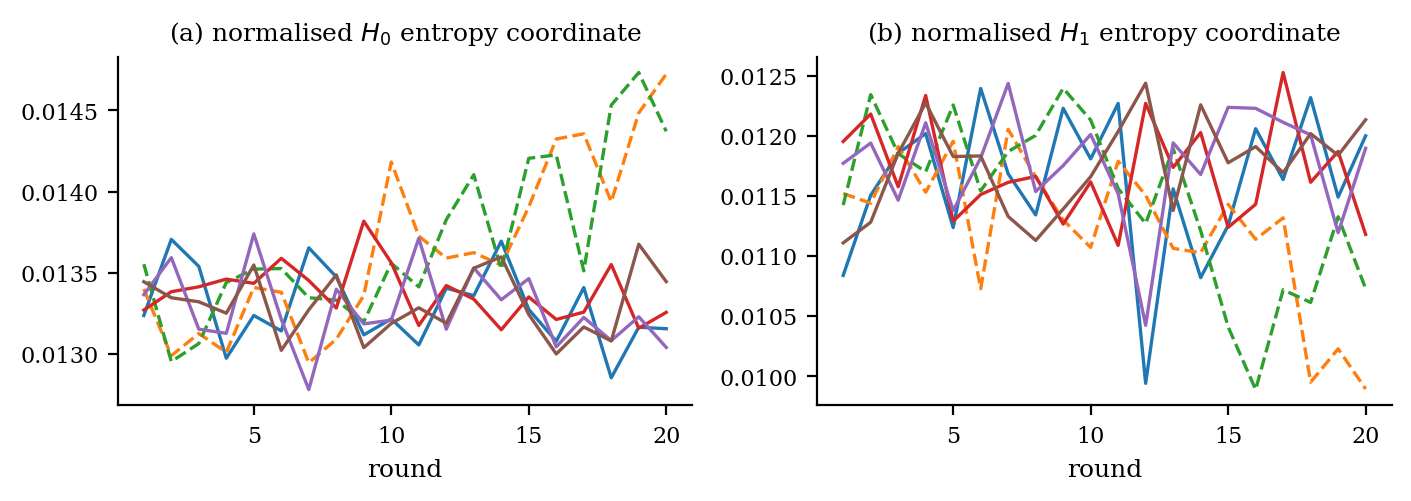}{0.8\textwidth}
\caption{Entropy coordinates of the normalised descriptor over $20$ rounds. Dashed lines are
the two drifting clients.}
\label{fig:continual}
\end{figure}

\section{Reproducibility}
\label{app:repro}
All results are produced by one script, \texttt{ptopofl\_v2.py}, which depends only on
NumPy, SciPy, scikit-learn and Matplotlib (Ripser is optional). Running
\texttt{python ptopofl\_v2.py --exp all --seeds 10} executes every experiment with seeds
$0$--$9$. The script writes the \LaTeX{} tables and the numerical macros used in the text,
so that the text and the code cannot diverge. The real dataset is downloaded from the UCI
repository, and the script prints the site sizes after preprocessing. On one CPU core the
synthetic experiments take about $12$ minutes in total.
\codeavailability

\bibliography{references}

\begin{thebibliography}{71}
\providecommand{\natexlab}[1]{#1}
\providecommand{\url}[1]{\texttt{#1}}
\expandafter\ifx\csname urlstyle\endcsname\relax
  \providecommand{\doi}[1]{doi: #1}\else
  \providecommand{\doi}{doi: \begingroup \urlstyle{rm}\Url}\fi

\bibitem[Abaach \& Morilla(2023)Abaach and Morilla]{abaach2023geotop}
Mariem Abaach and Ian Morilla.
\newblock {GeoTop}: Advancing image classification with geometric-topological
  analysis.
\newblock \emph{arXiv preprint arXiv:2311.16157}, 2023.
\newblock URL \url{https://arxiv.org/abs/2311.16157}.

\bibitem[Abadi et~al.(2016)Abadi, Chu, Goodfellow, McMahan, Mironov, Talwar,
  and Zhang]{abadi2016deep}
Martin Abadi, Andy Chu, Ian Goodfellow, H.~Brendan McMahan, Ilya Mironov, Kunal
  Talwar, and Li~Zhang.
\newblock Deep learning with differential privacy.
\newblock In \emph{Proceedings of the 2016 ACM SIGSAC Conference on Computer
  and Communications Security}, pp.\  308--318, 2016.

\bibitem[Adams et~al.(2017)Adams, Emerson, Kirby, Neville, Peterson, Shipman,
  Chepushtanova, Hanson, Motta, and Ziegelmeier]{adams2017persistence}
Henry Adams, Tegan Emerson, Michael Kirby, Rachel Neville, Chris Peterson,
  Patrick Shipman, Sofya Chepushtanova, Eric Hanson, Francis Motta, and Lori
  Ziegelmeier.
\newblock Persistence images: A stable vector representation of persistent
  homology.
\newblock \emph{Journal of Machine Learning Research}, 18\penalty0
  (8):\penalty0 1--35, 2017.

\bibitem[Alipourjeddi \& Miri(2025)Alipourjeddi and
  Miri]{alipourjeddi2025topology}
Narges Alipourjeddi and Ali Miri.
\newblock Topology-driven defense: Detecting model poisoning in federated
  learning with persistence diagrams.
\newblock In \emph{Proceedings of the 11th International Conference on
  Information Systems Security and Privacy (ICISSP)}, pp.\  389--396, 2025.
\newblock \doi{10.5220/0013367500003899}.

\bibitem[Atienza et~al.(2020)Atienza, Gonzalez-D{\'i}az, and
  Soriano-Trigueros]{atienza2020stability}
Nieves Atienza, Roc{\'i}o Gonzalez-D{\'i}az, and Manuel Soriano-Trigueros.
\newblock On the stability of persistent entropy and new summary functions for
  topological data analysis.
\newblock \emph{Pattern Recognition}, 107:\penalty0 107509, 2020.

\bibitem[Balle \& Wang(2018)Balle and Wang]{balle2018improving}
Borja Balle and Yu-Xiang Wang.
\newblock Improving the {G}aussian mechanism for differential privacy:
  Analytical calibration and optimal denoising.
\newblock In \emph{Proceedings of the 35th International Conference on Machine
  Learning}, volume~80 of \emph{PMLR}, pp.\  394--403, 2018.

\bibitem[Bauer(2021)]{ripser}
Ulrich Bauer.
\newblock Ripser: Efficient computation of {V}ietoris--{R}ips persistence
  barcodes.
\newblock \emph{Journal of Applied and Computational Topology}, 5\penalty0
  (3):\penalty0 391--423, 2021.
\newblock \doi{10.1007/s41468-021-00071-5}.

\bibitem[Blanchard et~al.(2017)Blanchard, El~Mhamdi, Guerraoui, and
  Stainer]{blanchard2017machine}
Peva Blanchard, El~Mahdi El~Mhamdi, Rachid Guerraoui, and Julien Stainer.
\newblock Machine learning with adversaries: {B}yzantine tolerant gradient
  descent.
\newblock In \emph{Advances in Neural Information Processing Systems},
  volume~30, 2017.

\bibitem[Bonawitz et~al.(2017)Bonawitz, Ivanov, Kreuter, Marcedone, McMahan,
  Patel, Ramage, Segal, and Seth]{bonawitz2017practical}
Keith Bonawitz, Vladimir Ivanov, Ben Kreuter, Antonio Marcedone, H.~Brendan
  McMahan, Sarvar Patel, Daniel Ramage, Aaron Segal, and Karn Seth.
\newblock Practical secure aggregation for privacy-preserving machine learning.
\newblock In \emph{Proceedings of the 2017 ACM SIGSAC Conference on Computer
  and Communications Security (CCS)}, pp.\  1175--1191. ACM, 2017.
\newblock \doi{10.1145/3133956.3133982}.

\bibitem[Boucheron et~al.(2013)Boucheron, Lugosi, and
  Massart]{boucheron2013concentration}
St{\'e}phane Boucheron, G{\'a}bor Lugosi, and Pascal Massart.
\newblock \emph{Concentration Inequalities: A Nonasymptotic Theory of
  Independence}.
\newblock Oxford University Press, 2013.

\bibitem[Briggs et~al.(2020)Briggs, Fan, and Andras]{briggs2020federated}
Christopher Briggs, Zhong Fan, and Peter Andras.
\newblock Federated learning with hierarchical clustering of local updates to
  improve training on non-{IID} data.
\newblock In \emph{International Joint Conference on Neural Networks (IJCNN)},
  pp.\  1--9, 2020.

\bibitem[Bubenik(2015)]{bubenik2015statistical}
Peter Bubenik.
\newblock Statistical topological data analysis using persistence landscapes.
\newblock \emph{Journal of Machine Learning Research}, 16\penalty0
  (3):\penalty0 77--102, 2015.

\bibitem[Charles \& Kone{\v{c}}n{\'y}(2021)Charles and
  Kone{\v{c}}n{\'y}]{charles2021convergence}
Zachary Charles and Jakub Kone{\v{c}}n{\'y}.
\newblock Convergence and accuracy trade-offs in federated learning and
  meta-learning.
\newblock In \emph{Proceedings of the 24th International Conference on
  Artificial Intelligence and Statistics}, volume 130 of \emph{PMLR}, pp.\
  2575--2583, 2021.

\bibitem[Chazal \& Michel(2021)Chazal and Michel]{chazal2021introduction}
Fr{\'e}d{\'e}ric Chazal and Bertrand Michel.
\newblock An introduction to topological data analysis: Fundamental and
  practical aspects for data scientists.
\newblock \emph{Frontiers in Artificial Intelligence}, 4:\penalty0 667963,
  2021.

\bibitem[Chazal et~al.(2009)Chazal, Cohen-Steiner, Guibas, M{\'e}moli, and
  Oudot]{chazal2009gromov}
Fr{\'e}d{\'e}ric Chazal, David Cohen-Steiner, Leonidas~J. Guibas, Facundo
  M{\'e}moli, and Steve~Y. Oudot.
\newblock Gromov--{H}ausdorff stable signatures for shapes using persistence.
\newblock \emph{Computer Graphics Forum}, 28\penalty0 (5):\penalty0 1393--1403,
  2009.

\bibitem[Chazal et~al.(2014{\natexlab{a}})Chazal, de~Silva, and
  Oudot]{chazal2014persistence}
Fr{\'e}d{\'e}ric Chazal, Vin de~Silva, and Steve Oudot.
\newblock Persistence stability for geometric complexes.
\newblock \emph{Geometriae Dedicata}, 173\penalty0 (1):\penalty0 193--214,
  2014{\natexlab{a}}.

\bibitem[Chazal et~al.(2014{\natexlab{b}})Chazal, Fasy, Lecci, Rinaldo, and
  Wasserman]{chazal2014stochastic}
Fr{\'e}d{\'e}ric Chazal, Brittany~Terese Fasy, Fabrizio Lecci, Alessandro
  Rinaldo, and Larry Wasserman.
\newblock Stochastic convergence of persistence landscapes and silhouettes.
\newblock In \emph{Proceedings of the 30th Annual Symposium on Computational
  Geometry (SoCG)}, pp.\  474--483, 2014{\natexlab{b}}.

\bibitem[Chen et~al.(2019)Chen, Ni, Bai, and Wang]{chen2019topological}
Chao Chen, Xiuyan Ni, Qinxun Bai, and Yusu Wang.
\newblock A topological regularizer for classifiers via persistent homology.
\newblock In \emph{Proceedings of the 22nd International Conference on
  Artificial Intelligence and Statistics (AISTATS)}, volume~89 of
  \emph{Proceedings of Machine Learning Research}, pp.\  2573--2582. PMLR,
  2019.
\newblock URL \url{https://proceedings.mlr.press/v89/chen19g.html}.

\bibitem[Chintakunta et~al.(2015)Chintakunta, Gentimis, Gonzalez-Diaz, Jimenez,
  and Krim]{chintakunta2015entropy}
Harish Chintakunta, Thanos Gentimis, Rocio Gonzalez-Diaz, Maria-Jose Jimenez,
  and Hamid Krim.
\newblock An entropy-based persistence barcode.
\newblock \emph{Pattern Recognition}, 48\penalty0 (2):\penalty0 391--401, 2015.

\bibitem[Clough et~al.(2022)Clough, Byrne, Oksuz, Zimmer, Schnabel, and
  King]{clough2020topological}
James~R. Clough, Nicholas Byrne, Ilkay Oksuz, Veronika~A. Zimmer, Julia~A.
  Schnabel, and Andrew~P. King.
\newblock A topological loss function for deep-learning based image
  segmentation using persistent homology.
\newblock \emph{IEEE Transactions on Pattern Analysis and Machine
  Intelligence}, 44\penalty0 (12):\penalty0 8766--8778, 2022.
\newblock \doi{10.1109/TPAMI.2020.3013679}.

\bibitem[Deng et~al.(2020)Deng, Kamani, and Mahdavi]{deng2020adaptive}
Yuyang Deng, Mohammad~Mahdi Kamani, and Mehrdad Mahdavi.
\newblock Adaptive personalized federated learning.
\newblock \emph{arXiv preprint arXiv:2003.13461}, 2020.

\bibitem[Dennis et~al.(2021)Dennis, Li, and Smith]{dennis2021heterogeneity}
Don~Kurian Dennis, Tian Li, and Virginia Smith.
\newblock Heterogeneity for the win: One-shot federated clustering.
\newblock In \emph{Proceedings of the 38th International Conference on Machine
  Learning}, volume 139 of \emph{PMLR}, pp.\  2611--2620, 2021.

\bibitem[Detrano et~al.(1989)Detrano, Janosi, Steinbrunn, Pfisterer, Schmid,
  Sandhu, Guppy, Lee, and Froelicher]{detrano1989international}
Robert Detrano, Andras Janosi, Walter Steinbrunn, Matthias Pfisterer,
  Johann-Jakob Schmid, Sarbjit Sandhu, Kern~H. Guppy, Stella Lee, and Victor
  Froelicher.
\newblock International application of a new probability algorithm for the
  diagnosis of coronary artery disease.
\newblock \emph{The American Journal of Cardiology}, 64\penalty0 (5):\penalty0
  304--310, 1989.

\bibitem[Dwork \& Roth(2014)Dwork and Roth]{dwork2014algorithmic}
Cynthia Dwork and Aaron Roth.
\newblock The algorithmic foundations of differential privacy.
\newblock \emph{Foundations and Trends in Theoretical Computer Science},
  9\penalty0 (3--4):\penalty0 211--407, 2014.
\newblock \doi{10.1561/0400000042}.

\bibitem[Edelsbrunner \& Harer(2010)Edelsbrunner and
  Harer]{edelsbrunner2010computational}
Herbert Edelsbrunner and John Harer.
\newblock \emph{Computational Topology: An Introduction}.
\newblock American Mathematical Society, 2010.

\bibitem[Geiping et~al.(2020)Geiping, Bauermeister, Dr{\"o}ge, and
  Moeller]{geiping2020inverting}
Jonas Geiping, Hartmut Bauermeister, Hannah Dr{\"o}ge, and Michael Moeller.
\newblock Inverting gradients -- how easy is it to break privacy in federated
  learning?
\newblock In \emph{Advances in Neural Information Processing Systems
  (NeurIPS)}, volume~33, pp.\  16937--16947. Curran Associates, Inc., 2020.

\bibitem[Ghosh et~al.(2020)Ghosh, Chung, Yin, and
  Ramchandran]{ghosh2020efficient}
Avishek Ghosh, Jichan Chung, Dong Yin, and Kannan Ramchandran.
\newblock An efficient framework for clustered federated learning.
\newblock In \emph{Advances in Neural Information Processing Systems
  (NeurIPS)}, volume~33, pp.\  19586--19597. Curran Associates, Inc., 2020.

\bibitem[Giusti et~al.(2015)Giusti, Pastalkova, Curto, and
  Itskov]{giusti2015clique}
Chad Giusti, Eva Pastalkova, Carina Curto, and Vladimir Itskov.
\newblock Clique topology reveals intrinsic geometric structure in neural
  correlations.
\newblock \emph{Proceedings of the National Academy of Sciences}, 112\penalty0
  (44):\penalty0 13455--13460, 2015.

\bibitem[Gouiaa et~al.(2024)Gouiaa, Vomo-Donfack, Tran-Dinh, and
  Morilla]{gouiaa2024taelcore}
Fatma Gouiaa, Kelly~L. Vomo-Donfack, Alexy Tran-Dinh, and Ian Morilla.
\newblock Novel dimensionality reduction method, {TaelCore}, enhances lung
  transplantation risk prediction.
\newblock \emph{Computers in Biology and Medicine}, 2024.
\newblock \doi{10.1016/j.compbiomed.2024.107969}.
\newblock Article 107969.

\bibitem[Hanzely \& Richt{\'a}rik(2020)Hanzely and
  Richt{\'a}rik]{hanzely2020federated}
Filip Hanzely and Peter Richt{\'a}rik.
\newblock Federated learning of a mixture of global and local models.
\newblock \emph{arXiv preprint arXiv:2002.05516}, 2020.

\bibitem[Hofer et~al.(2017)Hofer, Kwitt, Niethammer, and Uhl]{hofer2017deep}
Christoph Hofer, Roland Kwitt, Marc Niethammer, and Andreas Uhl.
\newblock Deep learning with topological signatures.
\newblock In \emph{Advances in Neural Information Processing Systems
  (NeurIPS)}, volume~30. Curran Associates, Inc., 2017.

\bibitem[Holm(1979)]{holm1979simple}
Sture Holm.
\newblock A simple sequentially rejective multiple test procedure.
\newblock \emph{Scandinavian Journal of Statistics}, 6\penalty0 (2):\penalty0
  65--70, 1979.

\bibitem[Hu et~al.(2025)Hu, Xiang, Tang, and Qiu]{hu2025fedtopo}
Ke~Hu, Liyao Xiang, Peng Tang, and Weidong Qiu.
\newblock {FedTopo}: Topology-informed representation alignment in federated
  learning under non-{I.I.D.} conditions.
\newblock \emph{arXiv preprint arXiv:2511.12628}, 2025.

\bibitem[Huber \& Ronchetti(2009)Huber and Ronchetti]{huber2009robust}
Peter~J. Huber and Elvezio~M. Ronchetti.
\newblock \emph{Robust Statistics}.
\newblock Wiley, 2nd edition, 2009.

\bibitem[Kairouz et~al.(2021)Kairouz, McMahan, et~al.]{kairouz2021advances}
Peter Kairouz, H.~Brendan McMahan, et~al.
\newblock Advances and open problems in federated learning.
\newblock \emph{Foundations and Trends in Machine Learning}, 14\penalty0
  (1--2):\penalty0 1--210, 2021.

\bibitem[Kang et~al.(2024)Kang, Kim, Sohn, and Awan]{kang2024differentially}
Taegyu Kang, Sehwan Kim, Jinwon Sohn, and Jordan Awan.
\newblock Differentially private topological data analysis.
\newblock \emph{Journal of Machine Learning Research}, 25\penalty0
  (189):\penalty0 1--42, 2024.
\newblock URL \url{https://jmlr.org/papers/v25/23-0585.html}.

\bibitem[Karimireddy et~al.(2020)Karimireddy, Kale, Mohri, Reddi, Stich, and
  Suresh]{karimireddy2020scaffold}
Sai~Praneeth Karimireddy, Satyen Kale, Mehryar Mohri, Sashank Reddi, Sebastian
  Stich, and Ananda~Theertha Suresh.
\newblock {SCAFFOLD}: Stochastic controlled averaging for federated learning.
\newblock In \emph{Proceedings of the 37th International Conference on Machine
  Learning (ICML)}, volume 119 of \emph{Proceedings of Machine Learning
  Research}, pp.\  5132--5143. PMLR, 2020.

\bibitem[Khaled et~al.(2020)Khaled, Mishchenko, and
  Richt{\'a}rik]{khaled2020tighter}
Ahmed Khaled, Konstantin Mishchenko, and Peter Richt{\'a}rik.
\newblock Tighter theory for local {SGD} on identical and heterogeneous data.
\newblock In \emph{Proceedings of the 23rd International Conference on
  Artificial Intelligence and Statistics}, volume 108 of \emph{PMLR}, pp.\
  4519--4529, 2020.

\bibitem[Koloskova et~al.(2020)Koloskova, Loizou, Boreiri, Jaggi, and
  Stich]{koloskova2020unified}
Anastasia Koloskova, Nicolas Loizou, Sadra Boreiri, Martin Jaggi, and
  Sebastian~U. Stich.
\newblock A unified theory of decentralized {SGD} with changing topology and
  local updates.
\newblock In \emph{Proceedings of the 37th International Conference on Machine
  Learning}, volume 119 of \emph{PMLR}, pp.\  5381--5393, 2020.

\bibitem[Li et~al.(2020)Li, Sahu, Zaheer, Sanjabi, Talwalkar, and
  Smith]{li2020federated}
Tian Li, Anit~Kumar Sahu, Manzil Zaheer, Maziar Sanjabi, Ameet Talwalkar, and
  Virginia Smith.
\newblock Federated optimization in heterogeneous networks.
\newblock \emph{Proceedings of Machine Learning and Systems}, 2:\penalty0
  429--450, 2020.

\bibitem[Li et~al.(2021)Li, Hu, Beirami, and Smith]{li2021ditto}
Tian Li, Shengyuan Hu, Ahmad Beirami, and Virginia Smith.
\newblock Ditto: Fair and robust federated learning through personalization.
\newblock In \emph{Proceedings of the 38th International Conference on Machine
  Learning}, volume 139 of \emph{PMLR}, pp.\  6357--6368, 2021.

\bibitem[Long et~al.(2023)Long, Xie, Shen, Zhou, Wang, and
  Jiang]{long2023multi}
Guodong Long, Ming Xie, Tao Shen, Tianyi Zhou, Xianzhi Wang, and Jing Jiang.
\newblock Multi-center federated learning: Clients clustering for better
  personalization.
\newblock \emph{World Wide Web}, 26\penalty0 (1):\penalty0 481--500, 2023.
\newblock \doi{10.1007/s11280-022-01046-x}.

\bibitem[Ma \& Gao(2024)Ma and Gao]{ma2024federated}
Zihan Ma and Tianchong Gao.
\newblock Federated learning backdoor attack detection with persistence
  diagram.
\newblock \emph{Computers \& Security}, 136:\penalty0 103557, 2024.

\bibitem[Malinovsky et~al.(2020)Malinovsky, Kovalev, Gasanov, Condat, and
  Richt{\'a}rik]{malinovsky2020local}
Grigory Malinovsky, Dmitry Kovalev, Elnur Gasanov, Laurent Condat, and Peter
  Richt{\'a}rik.
\newblock From local {SGD} to local fixed-point methods for federated learning.
\newblock In \emph{Proceedings of the 37th International Conference on Machine
  Learning}, volume 119 of \emph{PMLR}, pp.\  6692--6701, 2020.

\bibitem[McMahan et~al.(2017)McMahan, Moore, Ramage, Hampson, and Ag{\"u}era~y
  Arcas]{mcmahan2017communication}
Brendan McMahan, Eider Moore, Daniel Ramage, Seth Hampson, and Blaise
  Ag{\"u}era~y Arcas.
\newblock Communication-efficient learning of deep networks from decentralized
  data.
\newblock In \emph{Proceedings of the 20th International Conference on
  Artificial Intelligence and Statistics (AISTATS)}, volume~54 of
  \emph{Proceedings of Machine Learning Research}, pp.\  1273--1282. PMLR,
  2017.

\bibitem[Melis et~al.(2019)Melis, Song, De~Cristofaro, and
  Shmatikov]{melis2019exploiting}
Luca Melis, Congzheng Song, Emiliano De~Cristofaro, and Vitaly Shmatikov.
\newblock Exploiting unintended feature leakage in collaborative learning.
\newblock In \emph{IEEE Symposium on Security and Privacy (SP)}, pp.\
  691--706, 2019.

\bibitem[Mileyko et~al.(2011)Mileyko, Mukherjee, and
  Harer]{mileyko2011probability}
Yuriy Mileyko, Sayan Mukherjee, and John Harer.
\newblock Probability measures on the space of persistence diagrams.
\newblock \emph{Inverse Problems}, 27\penalty0 (12):\penalty0 124007, 2011.

\bibitem[Mironov(2017)]{mironov2017renyi}
Ilya Mironov.
\newblock R{\'e}nyi differential privacy.
\newblock In \emph{IEEE 30th Computer Security Foundations Symposium (CSF)},
  pp.\  263--275, 2017.

\bibitem[Nesterov(2018)]{nesterov2018lectures}
Yurii Nesterov.
\newblock \emph{Lectures on Convex Optimization}.
\newblock Springer, 2nd edition, 2018.

\bibitem[Ogier~du Terrail et~al.(2022)Ogier~du Terrail, Ayed, Cyffers,
  Grimberg, He, Loeb, Mangold, Marchand, Marfoq, Mushtaq,
  et~al.]{terrail2022flamby}
Jean Ogier~du Terrail, Samy-Safwan Ayed, Edwige Cyffers, Felix Grimberg,
  Chaoyang He, Regis Loeb, Paul Mangold, Tanguy Marchand, Othmane Marfoq, Erum
  Mushtaq, et~al.
\newblock {FLamby}: Datasets and benchmarks for cross-silo federated learning
  in realistic healthcare settings.
\newblock In \emph{Advances in Neural Information Processing Systems, Datasets
  and Benchmarks Track}, 2022.

\bibitem[Pathak \& Wainwright(2020)Pathak and Wainwright]{pathak2020fedsplit}
Reese Pathak and Martin~J. Wainwright.
\newblock {FedSplit}: An algorithmic framework for fast federated optimization.
\newblock In \emph{Advances in Neural Information Processing Systems},
  volume~33, 2020.

\bibitem[Reininghaus et~al.(2015)Reininghaus, Huber, Bauer, and
  Kwitt]{reininghaus2015stable}
Jan Reininghaus, Stefan Huber, Ulrich Bauer, and Roland Kwitt.
\newblock A stable multi-scale kernel for topological machine learning.
\newblock In \emph{Proceedings of the IEEE Conference on Computer Vision and
  Pattern Recognition (CVPR)}, pp.\  4741--4748, 2015.

\bibitem[Rieke et~al.(2020)Rieke, Hancox, Li, Milletar{\`i}, Roth, Albarqouni,
  Bakas, Galtier, Landman, Maier-Hein, Ourselin, Sheller, Summers, Trask, Xu,
  Baust, and Cardoso]{rieke2020future}
Nicola Rieke, Jonny Hancox, Wenqi Li, Fausto Milletar{\`i}, Holger~R. Roth,
  Shadi Albarqouni, Spyridon Bakas, Mathieu~N. Galtier, Bennett~A. Landman,
  Klaus Maier-Hein, S{\'e}bastien Ourselin, Micah Sheller, Ronald~M. Summers,
  Andrew Trask, Daguang Xu, Maximilian Baust, and M.~Jorge Cardoso.
\newblock The future of digital health with federated learning.
\newblock \emph{npj Digital Medicine}, 3\penalty0 (1):\penalty0 119, 2020.
\newblock \doi{10.1038/s41746-020-00323-1}.

\bibitem[Samuelson(1968)]{samuelson1968deviant}
Paul~A. Samuelson.
\newblock How deviant can you be?
\newblock \emph{Journal of the American Statistical Association}, 63\penalty0
  (324):\penalty0 1522--1525, 1968.

\bibitem[Sattler et~al.(2021)Sattler, M{\"u}ller, and
  Samek]{sattler2020clustered}
Felix Sattler, Klaus-Robert M{\"u}ller, and Wojciech Samek.
\newblock Clustered federated learning: Model-agnostic distributed multitask
  optimization under privacy constraints.
\newblock \emph{IEEE Transactions on Neural Networks and Learning Systems},
  32\penalty0 (8):\penalty0 3710--3722, 2021.

\bibitem[Sayed(2014)]{sayed2014adaptation}
Ali~H. Sayed.
\newblock Adaptation, learning, and optimization over networks.
\newblock \emph{Foundations and Trends in Machine Learning}, 7\penalty0
  (4--5):\penalty0 311--801, 2014.

\bibitem[Stich(2019)]{stich2019local}
Sebastian~U. Stich.
\newblock Local {SGD} converges fast and communicates little.
\newblock In \emph{International Conference on Learning Representations}, 2019.

\bibitem[T~Dinh et~al.(2020)T~Dinh, Tran, and Nguyen]{t2020personalized}
Canh T~Dinh, Nguyen Tran, and Josh Nguyen.
\newblock Personalized federated learning with {M}oreau envelopes.
\newblock In \emph{Advances in Neural Information Processing Systems
  (NeurIPS)}, volume~33, pp.\  21394--21405. Curran Associates, Inc., 2020.

\bibitem[Tan et~al.(2022)Tan, Long, Liu, Zhou, Lu, Jiang, and
  Zhang]{tan2022fedproto}
Yue Tan, Guodong Long, Lu~Liu, Tianyi Zhou, Qinghua Lu, Jing Jiang, and Chengqi
  Zhang.
\newblock {FedProto}: Federated prototype learning across heterogeneous
  clients.
\newblock In \emph{Proceedings of the AAAI Conference on Artificial
  Intelligence}, volume~36, pp.\  8432--8440, 2022.

\bibitem[Tauzin et~al.(2021)Tauzin, Lupo, Tunstall, P{\'e}rez, Caorsi,
  Medina-Mardones, Dassatti, and Hess]{tauzin2021giotto}
Guillaume Tauzin, Umberto Lupo, Lewis Tunstall, Julian~Burella P{\'e}rez,
  Matteo Caorsi, Anibal~M. Medina-Mardones, Alberto Dassatti, and Kathryn Hess.
\newblock giotto-tda: A topological data analysis toolkit for machine learning
  and data exploration.
\newblock \emph{Journal of Machine Learning Research}, 22\penalty0
  (39):\penalty0 1--6, 2021.

\bibitem[{The GUDHI Project}(2015)]{gudhi}
{The GUDHI Project}.
\newblock {GUDHI} library.
\newblock \url{https://gudhi.inria.fr}, 2015.
\newblock Version 3.x.

\bibitem[Tran-Dinh et~al.(2025)Tran-Dinh, Atchade, Tanaka, Lortat-Jacob,
  Castier, Mal, Messika, Mordant, Montravers, and
  Morilla]{tran2025topoattention}
Alexy Tran-Dinh, Enora Atchade, S{\'e}bastien Tanaka, Brice Lortat-Jacob, Yves
  Castier, Herv{\'e} Mal, Jonathan Messika, Pierre Mordant, Philippe
  Montravers, and Ian Morilla.
\newblock Early identification of high-risk individuals for mortality after
  lung transplantation: A retrospective cohort study with topological
  transformers.
\newblock \emph{medRxiv preprint 2025.10.01.25337124}, 2025.
\newblock URL
  \url{https://www.medrxiv.org/content/early/2025/10/03/2025.10.01.25337124}.

\bibitem[Turner et~al.(2014)Turner, Mileyko, Mukherjee, and
  Harer]{turner2014frechet}
Katharine Turner, Yuriy Mileyko, Sayan Mukherjee, and John Harer.
\newblock Fr{\'e}chet means for distributions of persistence diagrams.
\newblock \emph{Discrete \& Computational Geometry}, 52\penalty0 (1):\penalty0
  44--70, 2014.

\bibitem[Umeda(2017)]{umeda2017time}
Yuhei Umeda.
\newblock Time series classification via topological data analysis.
\newblock \emph{Transactions of the Japanese Society for Artificial
  Intelligence}, 32\penalty0 (3):\penalty0 D--G72\_1, 2017.
\newblock \doi{10.1527/tjsai.D-G72}.

\bibitem[Vahidian et~al.(2023)Vahidian, Morafah, Wang, Kungurtsev, Chen, Shah,
  and Lin]{vahidian2023efficient}
Saeed Vahidian, Mahdi Morafah, Weijia Wang, Vyacheslav Kungurtsev, Chen Chen,
  Mubarak Shah, and Bill Lin.
\newblock Efficient distribution similarity identification in clustered
  federated learning via principal angles between client data subspaces.
\newblock In \emph{Proceedings of the AAAI Conference on Artificial
  Intelligence}, volume~37, pp.\  10043--10052, 2023.
\newblock \doi{10.1609/aaai.v37i8.26197}.

\bibitem[Wei et~al.(2020)Wei, Li, Ding, Ma, Yang, Farokhi, Jin, Quek, and
  Poor]{wei2020federated}
Kang Wei, Jun Li, Ming Ding, Chuan Ma, Howard~H. Yang, Farhad Farokhi, Shi Jin,
  Tony Q.~S. Quek, and H.~Vincent Poor.
\newblock Federated learning with differential privacy: Algorithms and
  performance analysis.
\newblock \emph{IEEE Transactions on Information Forensics and Security},
  15:\penalty0 3454--3469, 2020.
\newblock \doi{10.1109/TIFS.2020.2988575}.

\bibitem[Woodworth et~al.(2020)Woodworth, Patel, and
  Srebro]{woodworth2020minibatch}
Blake Woodworth, Kumar~Kshitij Patel, and Nathan Srebro.
\newblock Minibatch vs local {SGD} for heterogeneous distributed learning.
\newblock In \emph{Advances in Neural Information Processing Systems},
  volume~33, 2020.

\bibitem[Yeom et~al.(2018)Yeom, Giacomelli, Fredrikson, and
  Jha]{yeom2018privacy}
Samuel Yeom, Irene Giacomelli, Matt Fredrikson, and Somesh Jha.
\newblock Privacy risk in machine learning: Analyzing the connection to
  overfitting.
\newblock In \emph{IEEE 31st Computer Security Foundations Symposium (CSF)},
  pp.\  268--282, 2018.

\bibitem[Yin et~al.(2018)Yin, Chen, Kannan, and Bartlett]{yin2018byzantine}
Dong Yin, Yudong Chen, Ramchandran Kannan, and Peter Bartlett.
\newblock {B}yzantine-robust distributed learning: Towards optimal statistical
  rates.
\newblock In \emph{Proceedings of the 35th International Conference on Machine
  Learning}, volume~80 of \emph{PMLR}, pp.\  5650--5659, 2018.

\bibitem[Zhao et~al.(2018)Zhao, Li, Lai, Suda, Civin, and
  Chandra]{zhao2018federated}
Yue Zhao, Meng Li, Liangzhen Lai, Naveen Suda, Damon Civin, and Vikas Chandra.
\newblock Federated learning with non-{IID} data.
\newblock \emph{arXiv preprint arXiv:1806.00582}, 2018.

\bibitem[Zhu et~al.(2019)Zhu, Liu, and Han]{zhu2019deep}
Ligeng Zhu, Zhijian Liu, and Song Han.
\newblock Deep leakage from gradients.
\newblock In \emph{Advances in Neural Information Processing Systems
  (NeurIPS)}, volume~32. Curran Associates, Inc., 2019.

\end{thebibliography}
\bibliographystyle{tmlr}

\end{document}